\documentclass[11pt, copyright, gdm]{google}

\usepackage[authoryear, sort&compress, round]{natbib}
\usepackage[utf8]{inputenc} 
\usepackage[T1]{fontenc}    
\usepackage{hyperref}       
\usepackage{url}            
\usepackage{booktabs}       
\usepackage{amsfonts,amsmath}       
\usepackage{nicefrac}       
\usepackage{microtype}      
\usepackage[svgnames]{xcolor}         
\usepackage{rotating}
\usepackage{booktabs,array}
\usepackage{tikz}
\usetikzlibrary{positioning,fit,calc,arrows.meta,backgrounds,shapes.multipart}
\definecolor{pandoraNavy}{RGB}{10,34,92}
\definecolor{pandoraBlue}{RGB}{77,128,220}
\definecolor{pandoraBlueFill}{RGB}{243,247,255}
\definecolor{pandoraGreen}{RGB}{46,133,59}
\definecolor{pandoraGreenFill}{RGB}{243,251,243}
\definecolor{pandoraGold}{RGB}{184,128,35}
\definecolor{pandoraGoldFill}{RGB}{252,248,239}
\definecolor{pandoraGray}{RGB}{165,169,180}
\definecolor{pandoraGrayFill}{RGB}{248,248,250}
\definecolor{pandoraText}{RGB}{40,40,40}
\uselogo{} 

\tikzset{
  >=Latex,
  pandora title/.style={
    font=\bfseries\fontsize{28}{30}\selectfont,
    text=pandoraNavy
  },
  pandora stage badge/.style={
    circle,
    fill=pandoraNavy,
    text=white,
    inner sep=0pt,
    minimum size=9mm,
    font=\bfseries\large
  },
  pandora stage title/.style={
    font=\bfseries\fontsize{20}{22}\selectfont,
    text=pandoraNavy
  },
  pandora divider/.style={
    draw=pandoraGray!75,
    dash pattern=on 1pt off 2pt,
    line width=0.6pt
  },
  pandora mainarrow/.style={
    -{Latex[length=3mm,width=2.2mm]},
    draw=pandoraNavy,
    line width=1.1pt
  },
  pandora thinarrow/.style={
    -{Latex[length=2.3mm,width=1.7mm]},
    draw=pandoraNavy,
    line width=0.9pt
  },
  pandora dashedarrow/.style={
    -{Latex[length=2.3mm,width=1.7mm]},
    draw=pandoraNavy,
    dash pattern=on 3pt off 2pt,
    line width=0.9pt
  },
  pandora prompt/.style={
    draw=pandoraNavy,
    rounded corners=2mm,
    line width=0.9pt,
    fill=white,
    align=center
  },
  pandora bluebox/.style={
    draw=pandoraBlue,
    fill=pandoraBlueFill,
    rounded corners=2mm,
    line width=0.9pt,
    align=center
  },
  pandora greenbox/.style={
    draw=pandoraGreen,
    fill=pandoraGreenFill,
    rounded corners=2mm,
    line width=0.9pt,
    align=center
  },
  pandora graybox/.style={
    draw=pandoraGray,
    fill=pandoraGrayFill,
    rounded corners=2mm,
    line width=0.8pt,
    align=center
  },
  pandora dashedbox/.style={
    draw=pandoraNavy,
    rounded corners=2mm,
    dash pattern=on 3pt off 2pt,
    line width=0.9pt,
    fill=white,
    align=center
  },
  pandora policy green/.style={
    draw=pandoraGreen,
    fill=pandoraGreenFill,
    rounded corners=2.2mm,
    line width=0.95pt,
    align=center
  },
  pandora policy gold/.style={
    draw=pandoraGold,
    fill=pandoraGoldFill,
    rounded corners=2.2mm,
    line width=0.95pt,
    align=center
  },
  pandora smallcap/.style={
    font=\fontsize{10.5}{12}\selectfont,
    text=pandoraText,
    align=center
  },
  pandora body/.style={
    font=\fontsize{11.5}{13}\selectfont,
    text=pandoraText,
    align=center
  },
  pandora policy title green/.style={
    font=\bfseries\fontsize{19}{21}\selectfont,
    text=pandoraGreen
  },
  pandora policy title gold/.style={
    font=\bfseries\fontsize{19}{21}\selectfont,
    text=pandoraGold
  },
  pandora policy subtitle green/.style={
    font=\itshape\fontsize{14}{16}\selectfont,
    text=pandoraGreen
  },
  pandora policy subtitle gold/.style={
    font=\itshape\fontsize{14}{16}\selectfont,
    text=pandoraGold
  }
}
\usepackage{float}

\usepackage{enumitem}
\usepackage[noend]{algpseudocode}
\usepackage{algorithm}
\usepackage{algorithmicx}
\usepackage{graphicx}
\definecolor{darkblue}{rgb}{0, 0, 0.5}
\hypersetup{colorlinks=true, citecolor=darkblue, linkcolor=darkblue, urlcolor=darkblue}
\usepackage{amsfonts,amsmath, amssymb, amsthm}

\newtheorem{proposition}{Proposition}
\newtheorem{corollary}{Corollary}
\newtheorem{remark}{Remark}

\usepackage{cleveref}
\usepackage{comment}
\usepackage{pgfplots}
\usepackage[parfill]{parskip}
\pgfplotsset{compat=1.18}
\usepgfplotslibrary{groupplots}
\usepackage{todonotes}
\usepackage{multirow}
\newcommand{\E}{\mathbb{E}}

\DeclareMathOperator*{\argmax}{arg\,max}
\newcommand{\calX}{\mathcal{X}}
\newcommand{\calY}{\mathcal{Y}}

\renewcommand{\paragraph}{\textbf}

\makeatletter
\newcommand{\multiline}[1]{%
  \parbox[t]{\dimexpr\linewidth-\ALG@thistlm}{#1\strut}%
}
\makeatother
\usepackage{etoc}
\usepackage[most,skins,theorems]{tcolorbox}

\tcbset{
  aibox/.style={
    width=\linewidth,
    top=6pt,
    bottom=0pt,
    colback=blue!6!white,
    colframe=black,
    colbacktitle=black,
    enhanced,
    center,
    attach boxed title to top left={yshift=-0.1in,xshift=0.15in},
    boxed title style={boxrule=0pt,colframe=white,},
  }
}
\newtcolorbox{AIbox}[2][]{aibox,title=#2,#1}

\tcbset{
    graybox/.style={
        width=\linewidth,
        top=6pt,
        bottom=0pt,
        colback=gray!20,
        colframe=black,
        colbacktitle=black,
        enhanced,
        center,
        attach boxed title to top left={yshift=-0.1in,xshift=0.15in},
        boxed title style={boxrule=0pt,colframe=white,},    
    }
}
\newtcolorbox{Graybox}[2][]{graybox,title=#2,#1}

\title{Pandora's AI Model Routing Box: Efficient Allocation with Costly Value Estimation}

\correspondingauthor{Correspondence to \texttt{\{fisch, jeisenstein\}@google.com.}}

\author{%
  Adam Fisch$^{*}$, Shubhendu Trivedi, Fantine Huot, William W. Cohen, Michael Kaisers, Mirella Lapata, Kate Larson, Jacob Eisenstein$^{*}$ \\
  Google DeepMind
}

\begin{document}

\begin{abstract}

Heterogeneous AI systems composed of multiple models, architectures, harnesses, or inference-time settings can improve quality and efficiency by routing queries to the \emph{specialist} who can answer most effectively at the lowest cost. Routing  requires estimating each specialist's expected return, but this value estimation has a cost. Cheap estimators (e.g., embedding-based predictors) are fast but noisy, while accurate estimators (e.g., fine-tuned models with access to retrieval results or partial reasoning traces) are expensive. We formalize this tradeoff as an instance of Pandora's Box, the classical problem of optimal search with costly inspection. Under a Gaussian signal model, the resulting policies have closed-form value-of-information expressions that determine, for each specialist and input, whether refining the value estimate is worth its cost. We call the centralized policy Pandora's Router.  We extend this to a decentralized setting, Pandora's Bidder, where specialists independently decide whether to invest in self-assessment before accepting an offered price to claim a query. Experiments across three domains---a standard multi-LLM benchmark, retrieval-augmented specialists, and LLMs with variable inference-time reasoning---show that Pandora's Router matches the routing quality of exhaustive estimation, while querying the expensive estimator far less often. In the decentralized setting, value-of-information reasoning improves allocative efficiency when competing estimates are accurate; when competing estimates are noisy, however, it can increase the strategic specialist's utility at the expense of others.\looseness=-1
\end{abstract}

\maketitle

\section{Introduction}
\label{sec:intro}
AI model providers now often offer heterogeneous model families spanning   a wide range of costs and capabilities: small and fast models for simple prompts, large and expensive ones for complex tasks, and augmented variants with  tools, retrieval, or extended reasoning for more specialized tasks. This raises the question of how to allocate each input to the model best suited for it, for a given level of cost sensitivity. This is the \emph{routing} problem, and a growing body of work addresses it by estimating each model's expected return on the input and selecting the maximizer~\citep{hu2024routerbench, shnitzer2023large, feng2026moco}. But \emph{value estimation} is neither free nor particularly easy. An embedding-based predictor is cheap but noisy; a fine-tuned scoring model is more accurate but more expensive; and computing partial solutions, executing tool calls, or running retrieval pipelines can be more expensive still. The routing decision itself thus involves a cost-accuracy tradeoff, which is typically ignored. We explore this tradeoff, asking: when is it worth paying for a better value estimate?\looseness=-1

The question has a clean analogy to the \emph{Pandora's Box} problem from search theory~\citep{weitzman1979optimal}. Pandora is presented with $M$ boxes, each with a hidden value which is known to her only in distribution. For a cost $c_m$ she can open any box $m$ and observe its value. At any point, she can stop searching and claim the best value seen thus far. Her goal is to maximize the difference between the value obtained minus the costs paid. \citet{weitzman1979optimal} showed that the optimal policy has a remarkably simple structure, based on computing, for each box, its \emph{reservation price}---the outside-option value at which the expected benefit of opening the box balances out the cost of opening it. Once these prices are calculated, Pandora opens the boxes in descending order of their reservation prices, stopping when she has found a value that exceeds the maximum of the reservation prices of the remaining unopened boxes.\looseness=-1

In this paper,  we cast routing with costly value estimation as an instance of Pandora's Box. Here, each specialist can be viewed as a box. The router always has access to a cheap value estimate $f_m(x)$, and can optionally pay $c_m$ to query a more accurate but costly estimate $g_m(x)$; i.e., pay to "open the box", and gain more information about the quality of specialist $m$. The question of which estimates to compute, and when to stop computing and commit to a specialist, is the same as Pandora's problem. Under a Gaussian signal model for the relationship between $f$ and $g$, the reservation prices and the associated value-of-information expressions have closed forms. To allow Pandora to select a box without making \emph{any} queries to $g$, we consider the \emph{non-obligatory} variant of the Pandora's Box problem, in which she can select a box that she has not opened~\citep{doval2018whether,beyhaghi2019pandora}.\looseness=-1

 Next, we consider the generalization from routing to decentralized forms of prompt allocation. In a routing system, a single centralized router controls all value estimation. In practice, however, the specialists themselves may be better positioned to estimate their own value. For example, a retrieval-augmented specialist has access to its own corpus, and can measure how relevant it is to the query; a math specialist can start to reason about or plan what computations will be required; a domain expert can know its own performance on internal benchmarks. None of these resources need be available to the router, and in some scenarios, specialists might even want to conceal information such as whether there are matches to a query in a private retrieval corpus. With this in mind, we extend the concept of value estimation for routing, to value estimation for \emph{bidding}, where decentralized agents learn how to leverage costly value estimation when claiming queries, in a way that maximizes their profit. Concretely, we propose \emph{Pandora's Bidder}: each specialist faces a posted price---the current best competing offer---and must decide whether to invest in a more accurate self-assessment before claiming the query. This corresponds to a single stage of the ascending-price mechanism with costly preference elicitation studied by \citet{parkes2005auction}; analyzing this simple setting isolates core components of decentralized prompt allocation under private value estimation with profit incentives.

\begin{AIbox}{Core contributions of this work}
As AI systems with diverse capabilities proliferate, reasoning about the economics of model selection becomes increasingly relevant. To our knowledge, this is the first paper to explicitly connect model routing with the classical Pandora's Box problem from economics and operations research. Specifically, this paper makes three main contributions:\looseness=-1
\begin{itemize}[leftmargin=*,itemsep=5pt]
\vspace{5pt}
    \item We formalize model routing with costly value estimation as an instance of Pandora's Box with non-obligatory inspection~\citep{doval2018whether}, yielding a reservation-price-based policy. Under a Gaussian signal model, this policy has a closed form that is straightforward to compute (\S\ref{sec:pandora-router}).\looseness=-1
    \item We extend to a decentralized setting where specialists use value-of-information reasoning to decide whether to refine their self-assessments before accepting a posted price (\S\ref{sec:pandora-bidder}).\looseness=-1
       \item We evaluate both frameworks  (see \Cref{fig:pandora-unified}) across three domains: a  multi-LLM routing benchmark (EmbedLLM), retrieval-augmented specialists for factoid QA, and inference-time computation for mathematical reasoning. In each setting, we find that both Pandora's Router and Pandora's Bidder are able to gracefully trade-off expensive and cheap estimates to achieve strong allocative performance relative to baselines across a range of estimation costs.\looseness=-1 
\end{itemize}
\end{AIbox}

\begin{figure*}[t]
\label{fig:pandora_high_level}
\centering
\resizebox{\textwidth}{!}{%
\begin{tikzpicture}[x=1cm,y=1cm]

\node[pandora stage badge, fill=white, draw=gray!50, text=gray!70!black, font=\small\mdseries, line width=0.6pt, inner sep=2pt] at (0.2,6.4) {1};
\node[pandora stage title, anchor=west, font=\normalsize\bfseries\color{gray!70!black}] at (1.1,6.4) {Prompt $x$};

\node[pandora prompt, minimum width=3.0cm, minimum height=1.8cm, text width=2.8cm, align=center] (promptbox) at (1.8,3.6) {Find the remainder\\when $7^{103}$ is\\divided by $13$.};


\draw[pandora mainarrow] (3.3,3.6) -- (4.4,3.6);

\draw[pandora divider] (3.8,-2.0) -- (3.8,6.8);

\node[pandora stage badge, fill=white, draw=gray!50, text=gray!70!black, font=\small\mdseries, line width=0.6pt, inner sep=2pt] at (4.3,6.4) {2};
\node[pandora stage title, anchor=west, font=\normalsize\bfseries\color{gray!70!black}] at (5.2,6.4) {Cheap screening};

\draw[draw=pandoraNavy,line width=1.1pt] (4.4,1.4) -- (4.4,5.0);
\draw[pandora thinarrow] (4.4,5.0) -- (5.2,5.0);
\draw[pandora thinarrow] (4.4,3.6) -- (5.2,3.6);
\draw[pandora thinarrow] (4.4,1.4) -- (5.2,1.4);

\node[pandora bluebox, minimum width=3.8cm, minimum height=1.0cm, align=center] (m1) at (7.2,5.0) {Specialist $1$\\[0.5mm]\scriptsize $f_1(x)$: modest prior};
\node[pandora bluebox, minimum width=3.8cm, minimum height=1.0cm, align=center] (m2) at (7.2,3.6) {Specialist $2$\\[0.5mm]\scriptsize $f_2(x)$: promising prior};
\node[font=\fontsize{24}{24}\selectfont,text=pandoraNavy] at (7.2,2.5) {$\vdots$};
\node[pandora bluebox, minimum width=3.8cm, minimum height=1.0cm, align=center] (mM) at (7.2,1.4) {Specialist $M$\\[0.5mm]\scriptsize $f_M(x)$: modest prior};

\node[pandora smallcap, text width=5.2cm, align=center] at (6.8,-0.8) {All specialists get an initial cheap, noisy value estimate $f_m(x)$.};

\draw[pandora mainarrow] (9.2,3.6) -- (10.5,3.6);

\draw[pandora divider] (9.6,-2.0) -- (9.6,6.8);

\node[pandora stage badge, fill=white, draw=gray!50, text=gray!70!black, font=\small\mdseries, line width=0.6pt, inner sep=2pt] at (10.1,6.4) {3};
\node[pandora stage title, anchor=west, font=\normalsize\bfseries\color{gray!70!black}] at (11.0,6.4) {Reservation prices};

\node[pandora dashedbox, minimum width=6.8cm, minimum height=1.4cm, align=left, text width=6.8cm] (rule) at (14.6,3.5) {%
Solve for the \textbf{reservation price} $u_m^\mathrm{rsv}$,\\ $\displaystyle \mathbb{E}_{g|f}\!\left[(g_m(x)-u_m^\mathrm{rsv})^+\right] = c_m$,
where
\begin{itemize}
\item $g_m(x)$ is a more accurate but costly value estimate;
\item $c_m$ is the cost of querying $g_m$.
\end{itemize}
};


\node[pandora smallcap, text width=8.0cm, align=center] at (14.6,-0.8) {Pandora estimates the value of information for running each expensive value estimator $g_m$.};

\draw[pandora mainarrow] (18.3,3.6) -- (19.4,3.6);

\draw[pandora divider] (18.8,-2.0) -- (18.8,6.8);

\node[pandora stage badge, fill=white, draw=gray!50, text=gray!70!black, font=\small\mdseries, line width=0.6pt, inner sep=2pt] at (19.3,6.4) {4};
\node[pandora stage title, anchor=west, font=\normalsize\bfseries\color{gray!70!black}] at (20.0,6.4) {Allocation};

\draw[draw=pandoraNavy,line width=1.1pt] (19.4,1.6) -- (19.4,4.8);
\draw[pandora thinarrow] (19.4,4.8) -- (20.0,4.8);
\draw[pandora thinarrow] (19.4,1.6) -- (20.0,1.6);

\node[pandora policy green, minimum width=5.5cm, minimum height=2.6cm] (router) at (22.8,4.4) {};
\node[font=\small\bfseries\color{pandoraGreen}] at (22.8,5.3) {Pandora's Router};
\node[font=\scriptsize\itshape\color{pandoraGreen!80!black}] at (22.8,4.95) {Centralized policy};
\draw[draw=pandoraGreen,line width=0.6pt] (20.3,4.75) -- (25.3,4.75);
\node[font=\scriptsize, text width=5.2cm, align=left] at (22.8,4) {Inspect specialists $g_m(x)$ in reservation-price order, stopping when no remaining query is worth its cost. Choose the maximizer of the observed $g_m$.};

\node[pandora policy gold, minimum width=5.5cm, minimum height=2.3cm] (bidder) at (22.8,1.6) {};
\node[font=\small\bfseries\color{pandoraGold}] at (22.8,2.5) {Pandora's Bidder};
\node[font=\scriptsize\itshape\color{pandoraGold!80!black}] at (22.8,2.15) {Decentralized policy};
\draw[draw=pandoraGold,line width=0.6pt] (20.3,1.95) -- (25.3,1.95);
\node[font=\scriptsize, text width=5.2cm, align=left] at (22.8,1.2) { The high bidder is selected in an auction. 
Each specialist refines its own value estimate only when the market price makes inspection worthwhile.};

\node[pandora smallcap, text width=6.2cm, align=center] at (22.8,-0.8) {The same  principle yields both a centralized and decentralized policy.};

\end{tikzpicture}%
}
\vspace{-10pt}
\caption{\textbf{Overview of the \textsc{Pandora} framework.} The system balances cheap screening and costly reasoning for routing. \textbf{(1)} A concrete prompt $x$ is received. \textbf{(2)} Multiple specialists $m \in \{1, \ldots, M\}$ generate cheap, noisy value estimates $f_m(x)$. \textbf{(3)} The system computes \textbf{reservation prices} that quantify the expected value-of-information for querying a more costly and accurate value estimate $g_m(x)$, e.g., by running the specialist for a few reasoning tokens and checking if it is on the right track. \textbf{(4)} This unified principle supports allocation by centralized routing and decentralized auctions.}
\label{fig:pandora-unified}
\vspace{-5pt}
\end{figure*}
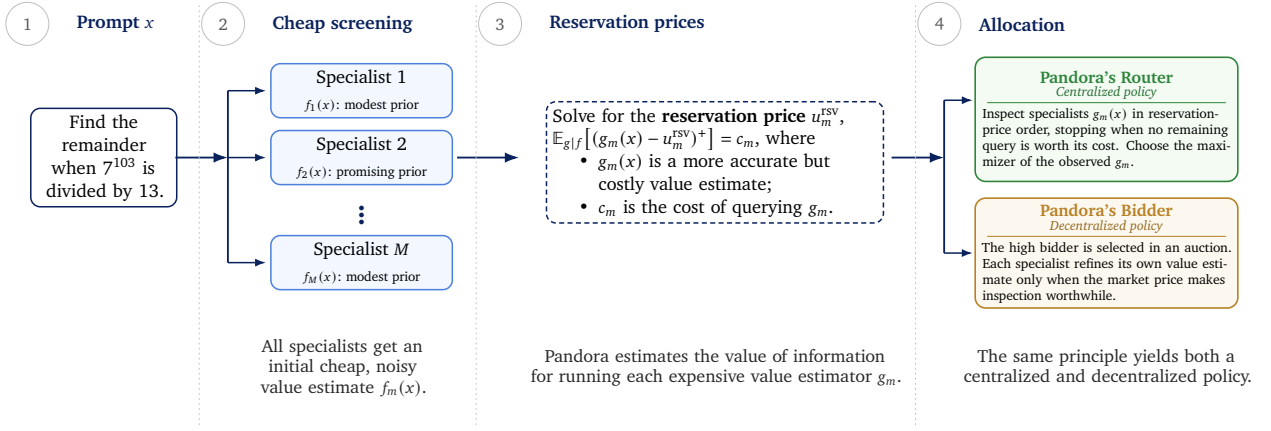
\section{Setting}
We begin with a more precise introduction to the problem of efficient allocation with costly value estimation. Suppose that a router must assign each prompt to the best specialist, but estimating specialist quality is itself costly. We assume that each specialist $m$ has a cheap value estimate $f_m(x)$ by default, while a more accurate but expensive estimate $g_m(x)$ can be queried selectively; we make these estimators concrete in \Cref{sec:experimental-domains}. The goal is to maximize the expected reward of the selected specialist minus total estimation cost by deciding which expensive estimates $g_m(x)$ are worth computing.\looseness=-1

Formally, let $X \in \calX$ be a random input prompt, and let $x$ denote a realization. 
Conditional on the input $X = x$, specialist $m$ samples an output $Y_m \sim P_m(\cdot \mid x)$, $Y_m \in \calY$, and receives a cost-adjusted reward $R_m(x, Y_m)$. For notational convenience, we will write the cost-adjusted reward $R_m(x, Y_m)$ simply as $R_m$. The reward may combine, for example, the graded quality of the output with the cost of producing it (which could be the inference cost, or API fees for tools called), and is a random variable through both the randomness of $Y_m$ and any noise in the evaluation itself (e.g., from a human annotator). The input-specific oracle \emph{routing} problem is to select the specialist that maximizes the conditional expected reward, that is, $\hat{m}(x) \in \argmax_m \rho_m(x),$ with $\rho_m(x) := \E[R_m \mid X=x].$ However, since the true reward is typically not available at routing time, the router must instead apply a \emph{value estimator} to predict it, and then pick the best specialist based on their estimated values. \looseness=-1

Of course, there are many ways to do value estimation, and these estimators may themselves incur widely varying computational costs. At one extreme, a cheap value estimate could just be a constant, such as the average cost-adjusted reward on a calibration set (and not conditional on $x$). At the other extreme, we might actually materialize the candidate output $Y_m = y$ for each model $m$, and then score $y$ with an estimate of $R_m$ (with an LLM-based auto-rater for example, if $R_m$ involves a human judgment), or even observe $R_m$ directly (if the reward is a known function of $y$). In practice, a straightforward approach to value estimation  is to apply a trained model to the text of the prompt, or its embedding~\cite[][\emph{inter alia}]{hu2024routerbench, lu-etal-2024-routing, ong2025routellm}. These estimators can move along a cost-accuracy curve depending on their complexity, or depending on any other  auxiliary input information that is gathered by the estimator (e.g., by invoking tools or computing partial solutions, such as plans~\citep{yao2022react,liang2025plantain}). We denote the additional information revealed from such a costly inspection of specialist $m$ by $Z_m$. Note that the costly information $Z_m$
need not be private to the specialist---it may simply reflect additional computation applied to the prompt, such as from using a more expressive encoder.\looseness=-1

Under this view, the routing process separates naturally into two phases: gathering information and making a final choice. We consider a minimal setting in which there are just two available value estimators: $f_m$ (cheap) and $g_m$ (costly), with costs $c_f < c_g$, respectively. For simplicity, we set $c_f = 0$, so that $f_m$ is always computed. We will use uppercase letters for the corresponding random variables. Let $F_m = f_m(X)$ be the cheap estimate for the specialist $m,$ and let $\mathbf F=(F_1, \ldots, F_M).$ 
Inspecting specialist $m$ reveals the costly information $Z_m$ needed to compute the refined estimate $G_m = g_m(X, Z_m)$; a process we simply call querying $g_m$. We denote the cost of this process as $c_g$.
For the additional cost of $G_m$ to be justified, it should be more accurate, in the sense that on  calibration data we have $\E[(R_m - G_m)^2] < \E[(R_m - F_m)^2].$ 
The router must dynamically decide which specialists to inspect (by querying $g_m$) in order to maximize the expected reward of the final assignment, minus the total cost of the realized $G_m$ estimates.\looseness=-1

To make things more concrete before presenting our efficient routing algorithms, we now describe three real-world domains where the two-tier estimation between $f_m$ versus $g_m$ arises organically, each with a qualitatively different source of costly information that is useful for deciding the best allocation.\looseness=-1

\subsection{Experimental domains and value estimators}
\label{sec:experimental-domains}

We explore several empirical settings, each featuring different types of specialists and value estimators. In all cases, the cheap estimator $f$ embeds the prompt with a pretrained encoder, retrieves the $k$ nearest prompts from a calibration set (measured by cosine similarity in embedding space), and returns the average reward of the retrieved neighbors as the value estimate.  The costly estimator is then constructed by fine-tuning a small language model encoder to predict the reward from a tailored, domain-specific input context. We describe these below. Note that in each domain, we partition the data into three splits: training (for fitting value estimators), calibration (for estimating the value-of-information model parameters in Sections~\ref{sec:pandora-router} and~\ref{sec:pandora-bidder}), and test. \Cref{tab:value-estimators} reports the MSE of each estimator on held-out calibration data. Additional dataset and implementation details are included in Appendix~\ref{app:dataset-details}.\looseness=-1

\paragraph{MATH: Inference-time scaling in mathematical reasoning tasks.}
In inference-time scaling, extended reasoning improves accuracy on hard problems but increases latency and cost. The costly estimator $g$ treats the early tokens of a reasoning trace as private information: it receives the prompt along with the first $t$ tokens of a model's chain-of-thought, effectively peeking at a partial solution to judge whether the model's reasoning looks promising before committing to it. This  is passed with the prompt to a small language model (Gemini 2.5 Flash-Lite) with a regression loss to predict the specialist's reward directly from the prompt text. 
The router chooses between Gemma3-4B~\citep{gemmateam2025gemma3technicalreport} (low-cost) and Gemini-3.1-Flash-Lite~\citep{gemini-3p1} (cost $0.66$, so that $R_m = \mathbf{1}\{Y_m \text{ is correct}\} - 0.66$). This high cost is chosen so that it is beneficial to route to Gemini only when confident that it can answer correctly and Gemma cannot. 
We evaluate on a corpus of 16{,}512 mathematical problems spanning MATH~\citep{data_HendrycksMATH}, Omni-Math~\citep{data_OmniMath}, AIME~\citep{data_aime_competition, data_olympiadbench}, and HMMT~\citep{data_hmmt_competition}.\looseness=-1

\paragraph{RAG: Retrieval-augmented generation with specialized corpora.}
In retrieval-augmented generation~\citep[RAG;][]{lewis2020retrieval}, retrieval results can improve answer quality but incur costs from retrieval computation, longer contexts, and potentially licensing fees for specialized corpora. The costly estimator $g$ runs retrieval and passes the results alongside the prompt to a language model fine-tuned with a regression loss, as in Math. Notably, the retrieval results themselves do carry some signal: if the retrieved documents are relevant to the query, the specialist is more likely to answer correctly, whereas irrelevant retrievals may hurt.
The router selects between three specialists: a low-cost model with no retrieval (the costly estimator $g$ here is just the more expensive LLM-based estimator), a Wikipedia RAG model, and a PubMed RAG model, where the RAG models each incur a cost of $0.05$ (again, subtracted from correctness in $R_m$).  
Prompts are a mixture of general-knowledge factoid questions and biomedical questions, following the same experimental setup as  \citet{eisenstein2025dont}.\looseness=-1

\paragraph{EmbedLLM: Large-scale language model selection.}
Finally, we consider EmbedLLM~\citep{zhuang2025embedllm}, a standard routing benchmark without auxiliary information, where the value of costly estimation comes purely from spending more compute to better distinguish among a large pool of candidates. 
As above, $g$ is a fine-tuned language model; even without private information, this is more expensive at inference time than the KNN baseline, but is significantly more accurate than the KNN baseline, since the model can learn prompt-level features that the embedding space misses.
EmbedLLM incorporates more than 100 open-weights models as routing targets; following \citet{jitkrittum2025universal}, we assign each model a cost proportional to its parameter count. Queries are drawn from standard benchmarks like MMLU~\citep{hendrycks2020measuring} and GSM8K~\citep{cobbe2021training}.

\begin{table}[t]
\small
    \centering
    \begin{tabular}{llllp{8.7cm}}
    \toprule
    Setting & Estimator & MSE & ${c_g}/{c_f}$ & Description\\
    \midrule
    MATH & $f$: KNN & $.154$ &\multirow{2}{*}{$5.8$} & $k$-nearest neighbors on prompt embeddings \\
    & $g$: SFT-CoT-$20$ & $.096$ & & Fine-tuned transformer on prompt + first $20$ tokens of the model's chain-of-thought (CoT) reasoning trace\\ \addlinespace[6pt]
    RAG & $f$: KNN & $.175$ & \multirow{2}{*}{$>7000$} & $k$-nearest neighbors on prompt embeddings \\
    & $g$: SFT-retrievals & $.109$& & Fine-tuned transformer on prompt + retrieval results\\[6pt]
    EmbedLLM & $f$: KNN & $.266$ & \multirow{2}{*}{$1.6$} & $k$-nearest neighbors on prompt embeddings \\
    & $g$: SFT-prompt & $.198$ && Fine-tuned transformer on prompt\\
    \bottomrule        
    \end{tabular}
        \caption{Value estimators across experimental domains. The cheap estimator $f$ is always computed; the costly estimator $g$ is queried selectively. The mean-squared error (MSE) is measured on a held-out calibration set. The cost ratios ${c_g}/{c_f}$ are discussed in \Cref{sec:setup-value-costs} and derived in \Cref{sec:supp-monetary-cost-derivations}.}
    \label{tab:value-estimators}
\end{table}

While the MATH and RAG settings have only two and three target models to choose from respectively, both routing and efficient value estimation are still challenging and important problems even when the number of target models is small. The strongest frontier models can be up to five times as expensive as cheaper models from the same provider,\footnote{Compare, e.g. Fable 5 vs Sonnet 5 at \url{https://platform.claude.com/docs/en/about-claude/pricing} (retrieved August 13, 2026).} so there are strong financial incentives to reserve the most expensive models for situations where they yield meaningful improvements. Similarly, as shown in \Cref{tab:value-estimators}, more computationally-intensive value estimation strategies can yield significant improvements, but at high cost. The difficulty of the information acquisition problem depends not primarily on the number of targets, but on how discriminable the options are: we can have many options, but this doesn’t matter if there is always a clear best box to use based on cheap signals alone; or we can have a few boxes that are hard to choose between, and it is expensive to know more. We picked experimental settings that try to cover key aspects of each of these considerations.

\subsection{Value estimator costs}
\label{sec:setup-value-costs}
To estimate monetary costs for the value estimators in \Cref{tab:value-estimators}, we use prices from \url{https://ai.google.dev/gemini-api/docs/pricing} (retrieved August 1, 2026), because it offers a single source of prices for LLM inference, embedding, and retrieval. The resulting estimates are meant only to show that approaches to value estimation can incur vastly different costs; we do not claim that these specific prices are optimal or even typical. Because the cost-accuracy tradeoff is a \emph{user} characteristic rather than an objectively-measurable property of the domain or method, we focus on the \emph{ratio} between the prices of $f$ and $g$. Please see \Cref{sec:supp-monetary-cost-derivations} for the derivation of these ratios.\looseness=-1
\section{Preliminaries: The Pandora's Box Problem}
\label{sec:pandora}
We begin by briefly reviewing the Pandora's Box problem~\citep{weitzman1979optimal}, which we will connect to model routing and model bidding in the following sections. In the Pandora's Box problem, a decision-maker (i.e., Pandora) is presented with $M$ options, called ``boxes'', each with an unknown reward (but with known distribution). For a price she can inspect any box to observe its reward, and at any time she can terminate the search and take the best reward that she has observed. The optimal policy depends on a mathematical  object called the \emph{reservation price}, which is a single scalar that encodes the value of inspecting a box. The reservation prices alone can be used to construct a simple, but optimal, priority-based search. Notably, this optimal search must be sequential, as the information acquired so far will help determine whether the remaining inspections still justify their cost.

To  gain intuition for the reservation price, suppose Pandora has already found a box with value $v$ for the input $x$, and must decide whether to open box $m$, whose hidden value is yet unknown, but known to be drawn from a distribution $P_m(\cdot \mid x)$. If Pandora opens the box and finds $V_m > v$, she takes $V_m$; otherwise, if $V_m < v$, she keeps $v$. Given an inspection cost of $c_m$, her expected payoff at this step from opening the box is therefore $\E[\max\{v, V_m\}] - c_m$; the payoff for \emph{not opening} $m$ is $v$. The net value is $\E[(V_m - v)^+] - c_m$, which is positive when the expected upside exceeds the cost. The \emph{reservation price} $u^{\text{rsv}}_m$, defined by \citet{weitzman1979optimal}, is the outside-option value at which this net value is exactly zero:\looseness=-1
\begin{equation}
\label{eq:reservation-price}
  \E[(V_m - u^{\text{rsv}}_m)^+] = c_m.
\end{equation}
When the current best value $v$ exceeds $u^{\text{rsv}}_m$, the cost of opening box $m$ outweighs the expected gain. The resulting decision rule is therefore quite simple: we  open box $m$ if $v < u^{\text{rsv}}_m$, and skip it otherwise.
The full selection strategy, however,  requires more than a single ``open-or-skip'' decision. Pandora must sequence her inspections, since each opened box updates the current best value $v$, which in turn changes. She may also skip inspection entirely: if the prior $P_m$ strongly favors box $m$, she can save $c_m$ and commit sight-unseen. These choices define the two classical variants of the problem.\looseness=-1

\paragraph{Pandora-OI (obligatory inspection).}
\label{sec:pandora-oi}  In the obligatory-inspection variant, Pandora must open every box she eventually selects. This is the classical Pandora's Box setting, and admits a simple solution. Specifically, \citet{weitzman1979optimal}  showed that reservation prices alone determine the optimal policy: initialize a best-seen value $v^* := -\infty$, open boxes in descending order of $u^{\text{rsv}}_m$, and stop as soon as $v^*$ exceeds the highest remaining reservation price. The box that yielded $v^*$ is selected. The remaining boxes can be safely ignored because $v^* > u^{\text{rsv}}_m$ implies $\E[(V_m - v^*)^+] < c_m$, by the definition of $u^{\text{rsv}}_m$.\looseness=-1

\paragraph{Pandora-NI (non-obligatory inspection).}
\label{sec:pandora-ni}
When inspection is non-obligatory, Pandora may select any unopened box, committing to it without paying the inspection cost, but accepting the risk that its realized value may disappoint~\citep{doval2018whether}. The optimal adaptive policy for this variant is NP-hard~\citep{fu2023pandora}, so we use a tractable restriction to \emph{committing policies}~\citep{beyhaghi2019pandora}. A committing policy designates a single box $m$ as ``held out'': this box will not be opened, but may be selected as the final choice. The remaining $M - 1$ boxes are searched via Pandora-OI, with the held-out box providing the initial outside option.\footnote{The value of the best committing policy that holds out multiple boxes cannot exceed that of the best single-box policy~\citep{beyhaghi2019pandora}, so it suffices to evaluate all $M$ single-holdout policies, in addition to the obligatory inspection policy (no holdouts). These M + 1 policies are compared by simulating their performance via MC sampling.}\looseness=-1

To use an unopened box as an outside option, we need to assign it a deterministic value. We use the \emph{backup price} $u^{\text{backup}}_m$, which like the reservation price in \Cref{eq:reservation-price}, is defined as the solution to:\looseness=-1
\begin{equation}
\label{eq:backup-price}
\E[(u^{\text{backup}}_m - V_m)^+] = c_m.
\end{equation}
Note that the backup price is related to, but different from, the reservation price. The reservation price is the outside option that makes Pandora indifferent to inspecting a box or skipping it; the backup price is the outside option that makes Pandora indifferent to committing to the sealed box sight-unseen versus paying to inspect it. Any opened box must exceed this threshold to justify discarding the sealed option. The relationship between  the backup price and the reservation price depends on the expected shortfall below the mean of the held-out box, specifically, $u^{\text{backup}}_m \le u^{\text{rsv}}_m$ if and only if $c_m \le \E[(\E[V_m] - V_m)^+]$. For   costs exceeding this threshold (very expensive boxes), the inequality reverses.\looseness=-1

The algorithm  proceeds as follows. For each candidate holdout $m \in \{0, 1, \ldots, M\}$, define the expected payoff of the committing policy that reserves box $m$:
\begin{equation}
\label{eq:committing-payoff}
\nu_m = \E\!\left[V_{\hat{m}} - \sum_{j \in O_m} c_j \right],
\end{equation}
where $V_{\hat{m}}$ is the value of the (random) box that the committing policy selects, and $O_m$ is the (random) set of boxes opened by Pandora-OI on $\{1,\ldots,M\}\setminus\{m\}$ with initial outside option $v^* = u^{\text{backup}}_m$. The case $m = 0$ corresponds to running Pandora-OI on all $M$ boxes with no holdout ($u^{\text{backup}}_0 = -\infty$). Because the held-out box is unobserved, we estimate $\nu_m$ via Monte Carlo (MC) samples $\tilde{V}_m \sim P_m(\cdot \mid x)$; we denote the MC estimate $\hat{\nu}_m$. We select  $m^*$ that maximizes $\hat{\nu}_m$, using $S=100$ samples.\footnote{Pilot experiments showed similar results with $S=30$, with performance degrading significantly only at $S=10$. We chose $S=100$ because the overall computational costs of this operation are cheap relative to value estimation itself.} At test time, we run Pandora-OI on $\{1, \ldots, M\} \setminus \{m^*\}$ with initial value $v^* = u^{\text{backup}}_{m^*}$, opening the actual boxes. We default to $m^*$ if we open no boxes or if no realized value clears $u_{m^*}^\text{backup}$; otherwise, we select the best opened box.\footnote{As a small technical detail, \citet{beyhaghi2019pandora} use the \emph{expected value} of the heldout box as the initial value $v^*$, rather than our choice of the backup price, which, empirically, we found to perform slightly better. See \Cref{sec:supplemental-algorithm}.}\looseness=-1

\section{Pandora's Router}
\label{sec:pandora-router}
A value-based router tries to select the specialist with the highest expected return. We connect routing to Pandora's Box by treating each specialist as a box and the costly value estimate as the value revealed by opening that box. One subtlety is that the box value is not the realized downstream reward $R_m$ itself.
The router never observes $R_m$ before choosing a specialist. The relevant decision value is thus the expected reward conditional on the information the router can acquire. In practice, our learned costly estimate $G_m$ is a learned plug-in approximation to this posterior decision value, and our proxy objective is to find the specialist for which $G_m$ is largest. In \Cref{prop:opened-value-reduction} and \Cref{cor:sealed-value} (Appendix~\ref{app:pandoras_objective}), we show that this is equivalent in expectation to optimizing $R_m$, assuming conditional independence.\looseness=-1

\paragraph{Gaussian signal model.}
The Pandora's Box algorithm requires an estimate of the distribution of the box values. We model the distribution of $G_m$ conditional on the cheap estimates as
\begin{equation}
\label{eq:gaussian_approx}
G_m \mid \mathbf{F} = \mathbf{f}\;\sim\; \mathcal{N}\!\left(\mu_m,\; \sigma_m^2\right), \qquad \mu_m = h_m(\mathbf{f})
\end{equation}
where $\mathbf{f}$ collects the realized cheap estimates $(f_1(x), \ldots, f_M(x))$ for all specialists and $\mu_m$ is a function $h_m(\cdot)$ of this vector, where $h_m$ and  $\sigma_m^2$ are estimated on calibration data. The choice of  $h_m$ is flexible; we explore both gradient boosted decision trees and linear regression depending on the domain (see \Cref{sec:implementation_details}). $\mu_m$ captures the predictable component of $g_m$ given $\mathbf{f}$, and $\sigma_m$ captures the residual uncertainty, which is approximated as normally distributed per \Cref{eq:gaussian_approx}.
Under this model, the value of opening box $m$ against an outside option $v$ is the expectation of a Gaussian censored at $v$:\looseness=-1
\begin{equation}
\label{eq:gaussian-voi}
\E[(G_m - v)^+] = (\mu_m - v)\,\Phi(\alpha_m) + \sigma_m\,\phi(\alpha_m),
\end{equation}
with $\alpha_m = (\mu_m - v)/\sigma_m$, and $\Phi$ and $\phi$ the standard normal CDF and PDF. The reservation price $u^{\text{rsv}}_m$ satisfies $\E[(G_m - u^{\text{rsv}}_m)^+] = c_m$ and can be obtained by simple root-finding on \Cref{eq:gaussian-voi}. Of course, while convenient, the Gaussian model is only an approximation, and real box values are rarely strictly Gaussian. We also explore a non-Gaussian signal model in \Cref{sec:non-gaussian-signal model} (for an alternative model based on Gaussian processes, see \citet{xie2024cost}), but find that while it is possible to improve fit to the empirical distribution, it does not yield an improvement in overall routing success and inspection cost.

\textbf{An extension for handling correlated values.} In routing, the Pandora-OI and Pandora-NI policies described above must contend with an additional complication: the hidden values of the boxes are not independent. For some difficult prompts, all specialist scores $G_m$ will tend to fall below the prediction $\mu_m$; for easy prompts, they will cluster above it. This correlation is particularly strong in the EmbedLLM domain, where there are $>100$ routing targets, some of which are extremely similar (e.g., different fine-tunings of the same base model).
To handle this, we propose a heuristic approximation that recomputes reservation prices at each step of the sequential policy (inspired by \citet{gergatsouli2023weitzman}, and also similar to the Gaussian process updates in \citet{xie2024cost}). On calibration data we estimate the parameters of a multivariate Gaussian model $\mathbf{G} \sim \mathcal{N}(\mu, \Sigma).$ After each box is opened, we condition on the observed values $\mathbf{G}_{\text{observed}}$ and compute $P(\mathbf{G}_{\text{unobserved}} \mid \mathbf{f}, \mathbf{G}_{\text{observed}})$ via the standard multivariate Gaussian posterior. We then make a mean-field approximation to this posterior, yielding updated marginal distributions for each remaining box, from which we recompute the reservation prices. The sequential policy then proceeds as before with these new parameters.

\begin{algorithm}[t]
\caption{Pandora's Router (with non-obligatory inspection)}
\label{alg:pandora-ni}
\small
\begin{algorithmic}[1]
\Require Cheap value estimates $\mathbf{f} = (f_1(x), \ldots, f_M(x))$, multi-variate Gaussian signal model parameters $(h, \Sigma)$, costs $\{c_m\}$, number of Monte Carlo samples $S$.
\State Compute predicted means $\mu_m \gets h(\mathbf{f})_m$ for all $m$
\State Compute reservation prices $u^{\text{rsv}}_m$ via root-finding on \Cref{eq:reservation-price}
\State Compute backup prices $u^{\text{backup}}_m$ for all $m$ via root-finding on \Cref{eq:backup-price}
\For{each candidate held-out box $m = 0, 1, \ldots, M$}
    \State \textcolor{DarkGreen}{// Note: $m = 0$ corresponds to Pandora-OI on $\{1, \ldots, M\}$ with $u_0^\mathrm{backup} = -\infty.$}
    \State \textcolor{DarkGreen}{// $\hat{\nu}_m$ is an MC estimate of the expected committing-policy payoff $\nu_m$ (\Cref{eq:committing-payoff}).}
    \State \multiline{Estimate $\hat{\nu}_m$ by running Pandora-OI on $\{1,\ldots,M\}\setminus\{m\}$ with initial value $v^* = u^{\text{backup}}_m$, averaged over $S$ samples of realized $\tilde{\mathbf{G}} = (\tilde{G}_1, \ldots, \tilde{G}_M)$ where $\tilde{\mathbf{G}} \sim \mathcal{N}(\mu, \Sigma)$.}
\EndFor
\State Select held-out box $m^* = \argmax_m \hat{\nu}_m$
\State Run Pandora-OI on $\{1,\ldots,M\}\setminus\{m^*\}$ with initial value $v^* = u^{\text{backup}}_{m^*}$, using actual $g$-queries
\If{no boxes were opened \textbf{or} the best realized value $\leq u^{\text{backup}}_{m^*}$}
    \State \Return sealed holdout $m^*$
\Else
    \State \Return the best opened box
\EndIf
\end{algorithmic}
\end{algorithm}

\paragraph{Experimental setup.} We evaluate Pandora's Router on the three domains described in \S\ref{sec:experimental-domains} (MATH, RAG, and EmbedLLM), and compare with the following methods:
\begin{itemize}[leftmargin=*,itemsep=5pt]
\item \textbf{$\mathbf{f}$-only}: Route using the cheap estimator $f_m$ only (never opening any box).
\item \textbf{$\mathbf{g}$-always}: Route using the expensive estimator $g_m$ only (always opening all boxes).
\item \textbf{Top-$\mathbf{2}$}: Always queries $g_m$ for the two models with highest $f_m$ (open the two expected best boxes).
\item \textbf{Coin Flip}: Query $g_m$ with probability $\frac{1}{2}$; route to the maximum observed (randomly open boxes).

\end{itemize}

We also evaluate two ablations that use the same inspection budget selected by Pandora's Router, but use different methods to decide which boxes to open (that is, instead of reservation prices). Specifically, we first run Pandora's Router to count how many times, $N_\mathrm{pr}$, in total $g_m$ was queried  over a test set (in hindsight), and then reallocate those $N_\mathrm{pr}$ inspections according to the following rules:\looseness=-1
\begin{itemize}[leftmargin=*, itemsep=5pt]
\item \textbf{Random-$\mathbf{N_{\mathrm{pr}}}$:} A simple control where $g_m$ is randomly queried up to $N_\mathrm{pr}$ times over the entire test set; for every example we route to the best candidate among the randomly inspected options. We route to $\argmax_m f_m$ as a default if no inspections were made for that example.
    \item \textbf{Margin-$\mathbf{N_{\mathrm{pr}}}$}: An uncertainty-based heuristic that queries $g_m$ for specialists whose $f$-scores are close to being best: letting $m' = \argmax_j f_j$, if $f_{m'} - f_m$ is small, we inspect both $g_m$ and $g_{m'}$. Again, we inspect $g_m$ for at most $N_\mathrm{pr}$ specialists over the entire test set, and route to $m'$ as a default.\looseness=-1
\end{itemize}
The primary metric is \emph{regret $+$ inspection cost}: routing regret (the gap between the selected specialist's true, realized reward and the oracle best specialist in hindsight) plus the total cost of all $g$-queries. Although our method generalizes to heterogeneous costs, for simplicity we use a uniform inspection cost $c_m = c_g$ $\forall m \in \{1, \ldots, M\}$, and sweep $c_g$ to trace out the cost versus performance frontier.

\begin{table}[t]
\centering
\small
\begin{tabular}{l|ccc|ccc|ccc}
\toprule
\multirow{2}{*}{\textbf{Method}} & \multicolumn{3}{c|}{MATH} & \multicolumn{3}{c|}{RAG} & \multicolumn{3}{c}{EmbedLLM} \\
&\textbf{Regret} & \textbf{Cost} & \textbf{Total} & \textbf{Regret} & \textbf{Cost} & \textbf{Total} & \textbf{Regret} & \textbf{Cost} & \textbf{Total}\\
\midrule
\text{$f$-only} & 0.117 & 0.000 & 0.117 & 0.150 & 0.000 & 0.150 & 0.393 & 0.000 & 0.393 \\
\text{$g$-only} & 0.090 & 0.038 & 0.128 & 0.084 & 0.057 & 0.141 & 0.370 & 1.986 & 2.356 \\
\text{Top-$2$} & 0.090 & 0.038 & 0.128 & 0.107 & 0.038 & 0.146 & 0.402 & 0.036 & 0.438 \\
\text{Coin Flip} & 0.107 & 0.019 & 0.127 & 0.122 & 0.029 & 0.151 & 0.398 & 0.992 & 1.390 \\
\cmidrule(lr){1-10}
\text{Random-$N_{\mathrm{pr}}$} & 0.164 & 0.011 & 0.175 & 0.135 & 0.027 & 0.162 & 0.363 & 0.075 & 0.438 \\
\text{Margin-$N_{\mathrm{pr}}$} & 0.094 & 0.011 & 0.105 & 0.101 & 0.027 & 0.128 & 0.314 & 0.075 & 0.389 \\
\cmidrule(lr){1-10}
\text{Pandora's Router} & 0.094 & 0.011 & \textbf{0.105} & 0.091 & 0.027 & \textbf{0.118} & 0.311 & 0.075 & \textbf{0.386} \\
\bottomrule
\end{tabular}
\caption{Routing evaluation results averaged across query costs $c_g$. Lower is better on all metrics; the lowest average regret + cost measured per setting is \textbf{bolded}. See  \Cref{sec:supplemental-numerical-significance} in the Appendix for a breakdown of results  per cost level $c_g$, together with paired statistical significance tests.}
\label{tab:overall-results}
\end{table}

\paragraph{Results.} 
Results for the full set of baselines  are shown in \Cref{tab:overall-results} for all three datasets, averaging over all values of $c_g$ (see \Cref{sec:supplemental-numerical-significance} for results per $c_g$; the values tested correspond to those in Figure~\ref{fig:results}). Pandora's Router is reliably the best. The two inspection ablations, \text{Margin-$N_{\mathrm{pr}}$} and \text{Random-$N_{\mathrm{pr}}$}, borrow the query budget from Pandora's Router, but achieve higher regret because they do not use the budget as effectively. Differences in particular between \text{Margin-$N_{\mathrm{pr}}$}, which is the most competitive comparison, and Pandora's Router at each specific cost level are  shown in \Cref{fig:results} (for clarity in the figure, we compare Pandora's Router with $f$-only, $g$-only, and the margin baseline). 
Pandora minimizes regret plus inspection cost at nearly all cost levels $c_g$, in all settings. When the cost of $g$ is low, Pandora's Router queries it for nearly every specialist, and all methods except $f$-only (which never queries $g$) perform similarly. As $c_g$ increases, the $g$-always baseline continues to make a fixed number of queries regardless of cost, while $f$-only still never queries at all. Pandora's Router interpolates between these extremes, querying $g$ only when the value of information exceeds the cost. As a result, the total regret + inspection cost of Pandora's Router tracks the lower envelope of the baselines across the full $c_g$ range.\looseness=-1

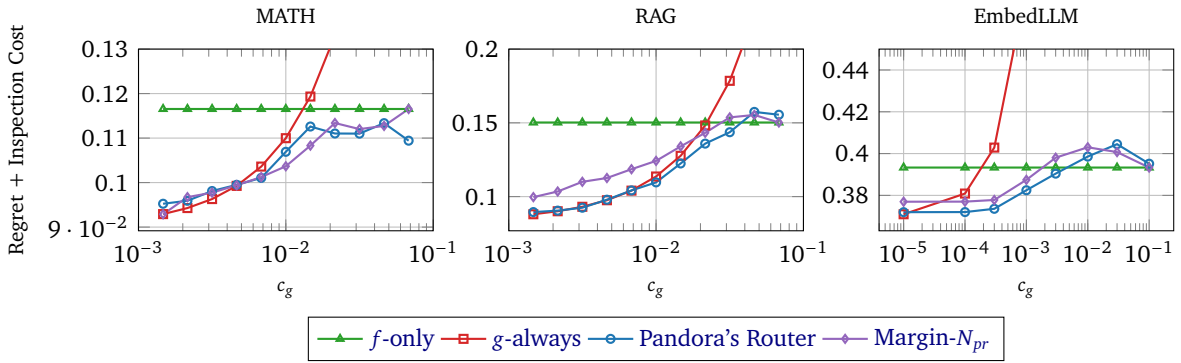
\begin{figure}
\begin{tikzpicture}
  \begin{groupplot}[
    group style={
      group size=3 by 1,
      horizontal sep=1cm,
      x descriptions at=edge bottom
    },
    width=0.33\textwidth,
    height=0.24\textwidth,
    xlabel={$c_g$},
    grid=major,
    label style={font=\scriptsize},
    tick label style={font=\footnotesize},
    title style={font=\scriptsize},
    xmode=log,
  ]

    \nextgroupplot[title={MATH}, ymax=0.13, ylabel={Regret + Inspection Cost}, legend to name=sharedlegend, legend columns=4, legend style={font=\footnotesize}]
    \addplot[color={rgb,255:red,44;green,160;blue,44}, mark=triangle, mark size=1.5, thick] coordinates {
      (0.00147, 0.116539) (0.00215, 0.116539) (0.00316, 0.116539) (0.00464, 0.116539) (0.00681, 0.116539) (0.01, 0.116539) (0.01468, 0.116539) (0.02154, 0.116539) (0.03162, 0.116539) (0.04642, 0.116539) (0.06813, 0.116539)
    }; \addlegendentry{$f$-only}
    \addplot[color={rgb,255:red,214;green,39;blue,40}, mark=square, mark size=1.5, thick] coordinates {
      (0.00147, 0.0929178) (0.00215, 0.0942778) (0.00316, 0.0962978) (0.00464, 0.0992578) (0.00681, 0.103598) (0.01, 0.109978) (0.01468, 0.119338) (0.02154, 0.133058) (0.03162, 0.153218) (0.04642, 0.182818) (0.06813, 0.226238)
    }; \addlegendentry{$g$-always}
    \addplot[color={rgb,255:red,31;green,119;blue,180}, mark=o, mark size=1.5, thick] coordinates {
      (0.00147, 0.095244) (0.00215, 0.0959155) (0.00316, 0.0981051) (0.00464, 0.0994982) (0.00681, 0.101053) (0.01, 0.106919) (0.01468, 0.112598) (0.02154, 0.111031) (0.03162, 0.111003) (0.04642, 0.113357) (0.06813, 0.109417)
    }; \addlegendentry{Pandora's Router}
    \addplot[color={rgb,255:red,148;green,103;blue,189}, mark=diamond, mark size=1.5, thick] coordinates {
      (0.00147, 0.0928829) (0.00215, 0.0967321) (0.00316, 0.0977329) (0.00464, 0.0993982) (0.00681, 0.101242) (0.01, 0.103664) (0.01468, 0.108309) (0.02154, 0.113386) (0.03162, 0.11202) (0.04642, 0.112735) (0.06813, 0.116539)
    }; \addlegendentry{Margin-$N_{pr}$}

    \nextgroupplot[title={RAG}, ymax=0.2]
    \addplot[color={rgb,255:red,44;green,160;blue,44}, mark=triangle, mark size=1.5, thick] coordinates {
      (0.00147, 0.150234) (0.00215, 0.150234) (0.00316, 0.150234) (0.00464, 0.150234) (0.00681, 0.150234) (0.01, 0.150234) (0.01468, 0.150234) (0.02154, 0.150234) (0.03162, 0.150234) (0.04642, 0.150234) (0.06813, 0.150234)
    };
    \addplot[color={rgb,255:red,214;green,39;blue,40}, mark=square, mark size=1.5, thick] coordinates {
      (0.00147, 0.088035) (0.00215, 0.090075) (0.00316, 0.093105) (0.00464, 0.097545) (0.00681, 0.104055) (0.01, 0.113625) (0.01468, 0.127665) (0.02154, 0.148245) (0.03162, 0.178485) (0.04642, 0.222885) (0.06813, 0.288015)
    };
    \addplot[color={rgb,255:red,31;green,119;blue,180}, mark=o, mark size=1.5, thick] coordinates {
      (0.00147, 0.0895061) (0.00215, 0.0904973) (0.00316, 0.0924703) (0.00464, 0.0979087) (0.00681, 0.104191) (0.01, 0.109644) (0.01468, 0.12246) (0.02154, 0.135927) (0.03162, 0.14365) (0.04642, 0.15742) (0.06813, 0.1555)
    };
    \addplot[color={rgb,255:red,148;green,103;blue,189}, mark=diamond, mark size=1.5, thick] coordinates {
      (0.00147, 0.0996467) (0.00215, 0.103466) (0.00316, 0.110033) (0.00464, 0.112612) (0.00681, 0.118566) (0.01, 0.124222) (0.01468, 0.134038) (0.02154, 0.143224) (0.03162, 0.153696) (0.04642, 0.155358) (0.06813, 0.150234)
    };

    \nextgroupplot[title={EmbedLLM}, ymax=0.45]
    \addplot[color={rgb,255:red,44;green,160;blue,44}, mark=triangle, mark size=1.5, thick] coordinates {
      (1e-05, 0.393312) (0.0001, 0.393312) (0.0003, 0.393312) (0.001, 0.393312) (0.003, 0.393312) (0.01, 0.393312) (0.03, 0.393312) (0.1, 0.393312) 
    };
    \addplot[color={rgb,255:red,214;green,39;blue,40}, mark=square, mark size=1.5, thick] coordinates {
      (1e-05, 0.370964) (0.0001, 0.380864) (0.0003, 0.402864) (0.001, 0.479864) (0.003, 0.699864) (0.01, 1.46986) (0.03, 3.66986) (0.1, 11.3699) 
    };
    \addplot[color={rgb,255:red,31;green,119;blue,180}, mark=o, mark size=1.5, thick] coordinates {
      (1e-05, 0.371963) (0.0001, 0.372009) (0.0003, 0.373577) (0.001, 0.382453) (0.003, 0.390407) (0.01, 0.398565) (0.03, 0.404505) (0.1, 0.395146) 
    };
    \addplot[color={rgb,255:red,148;green,103;blue,189}, mark=diamond, mark size=1.5, thick] coordinates {
      (1e-05, 0.376963) (0.0001, 0.377008) (0.0003, 0.377793) (0.001, 0.38743) (0.003, 0.398037) (0.01, 0.403009) (0.03, 0.400709) (0.1, 0.393312) 
    };

  \end{groupplot}

  \node[below=1cm] at ($(group c2r1.south)$) {\ref{sharedlegend}};
\end{tikzpicture}
\vspace{-5pt}
\caption{Routing performance on the MATH, RAG, and EmbedLLM domains for varying costs of querying $g$, with $f = \text{KNN}_3$. The ideal value estimation policy would minimize regret + inspection cost at every cost level.
Pandora's Router achieves near-minimal total cost across the full range of $c_g$ (where $c_m = c_g$ for all $m$), querying $g$ frequently when it is cheap and rarely when it is expensive.\looseness=-1}
\label{fig:results}
\vspace{-5pt}
\end{figure}

As a secondary analysis, \Cref{fig:monetary-cost-vs-regret} in \Cref{sec:supp-monetary-cost-derivations} shows the cost-performance tradeoff offered by Pandora's Router, in terms of monetary costs and routing regret. The costs are computed from the Gemini API prices described in \Cref{sec:setup-value-costs}. 
When $c_g=0.001$, we are in a setting in which users are eager to pay for queries that improve routing; here Pandora's Router obtains regret that almost matches $g$-only, while incurring much lower inspection costs.
When $c_g=0.1$, we are in a setting in which users are unwilling to pay for queries to $g$; here Pandora's Router never queries $g$, and routing performance and cost match the $f$-only baseline. These tradeoffs are dynamically navigated by Pandora's Router.\looseness=-1

Note that in  the MATH experiment the aggregate results for Pandora's Router and \text{Margin-$N_\mathrm{pr}$} are nearly identical. Since there are only two routing targets, regret is determined largely by the decision about how many $g$ values to query: with two queries, both methods always select the $g$-maximizer; with zero queries both methods almost always select the $f$-maximizer (except in rare cases where the reserve price ordering is different from the ordering of $f$). Because \text{Margin-$N_\mathrm{pr}$} inherits the query budget ($N_\mathrm{pr}$) from Pandora's Router, it is unsurprising that its overall performance is very similar. \looseness=-1
\section{Pandora's Bidder}
\label{sec:pandora-bidder}

We now derive a simple but interesting extension of Pandora's Router to a decentralized setting in which the specialists control their own value estimates, and use them to participate in a marketplace. Pandora's Router assumes that a single decision-maker controls which boxes to open and which specialist to select. In many  settings, however, the specialists themselves are better positioned to estimate their own value:  a retrieval-augmented specialist has access to its own corpus, a math specialist can execute partial computations, and a domain expert may have proprietary benchmarks. None of these resources need be visible to a centralized router. Decentralization also has a number of practical advantages: new specialists can join without retraining the routing model, and the cost of value estimation is borne by the specialists rather than the platform.\looseness=-1

As an alternative to a centralized router, consider a market-based mechanism where each specialist can purchase the right to answer a query. We formalize this via a posted-price mechanism corresponding to a single stage of the ascending-price framework with costly preference elicitation studied by \citet{parkes2005auction}.  We consider a leave-one-out setting with a single strategic specialist. A platform collects nonstrategic value estimates from $M - 1$ specialists and posts the best estimate as the price at which the remaining specialist can claim the query. The strategic specialist must then apply the same value-of-information (VoI) reasoning from Pandora's Router to decide whether this price is worth accepting, and whether or not to pay to observe a refined estimate before making this decision.

Concretely, for a given query $x$, the platform solicits $G_j$ from each of the $M - 1$ nonstrategic specialists and sets the posted price as $p = \max_{j \neq m} G_j$. Note that the platform could incorporate a markup to extract profit, but here we simply use the best competing estimate. The strategic specialist $m$ then faces a decision: accept the price (winning the right to answer the query, instead of the platform's choice), or decline (in which case the platform routes to the specialist that set the price). Specialist $m$ begins with only the cheap estimate, which yields a predicted mean $\mu_m(x) = h(\mathbf{f})_m$ for $G_m$, as in \Cref{eq:gaussian_approx}. The expected gain from paying $c_m$ to observe $G_m$ before responding to the price $p$ is\looseness=-1
\begin{equation}
\label{eq:bidding_gain}
    \text{VoI}(p) = \mathbb{E}[(G_m - p)^+] - (\mu_m - p)^+.
\end{equation}
This is the difference in expected profit between deciding after observing $G_m$ and deciding based only on $\mu_m$. The structure mirrors the reservation price computation for Pandora's Router in \S\ref{sec:pandora-router}, but with the price $p$ playing the role of the outside option. When $p$ is low relative to $\mu_m$, the specialist can confidently accept based on $f$ alone; when $p$ is high, the specialist can confidently decline. The value of information peaks when $p$ is close to $\mu_m$ and the correct action is uncertain.
The condition $\text{VoI}(p) = c_m$ defines an \emph{interval} $[p_{\text{lo}}, p_{\text{hi}}]$ within which it is worth paying for the refined estimate~\citep{parkes2005auction}. Using the same Gaussian model as in \Cref{eq:gaussian_approx,eq:gaussian-voi}, the boundaries satisfy\looseness=-1
\begin{equation}
\label{eq:gaussian_approx_bidding}
(\mu_m - p)\,\Phi(\alpha_m) + \sigma_m\,\phi(\alpha_m) = c_m + (\mu_m - p)^+
\end{equation}
with $\alpha_m = (\mu_m - p)/\sigma_m$, and $\Phi$ and $\phi$ the standard normal CDF and PDF. Like the reservation price, these values are obtained via standard root-finding techniques. The resulting bidding strategy pays $c_m$ to observe $G_m$ when $p_{\text{lo}} \leq p \leq p_{\text{hi}}$, accepting if $G_m > p$. Outside this interval, refinement is not worth it: the specialist  simply accepts if $p < p_{\text{lo}}$ and declines if $p >p_{\text{hi}}$.\looseness=-1

\begin{figure}[t]
\centering
\begin{tikzpicture}
\begin{groupplot}[
  group style={
    group size=3 by 2,
    horizontal sep=1cm,
    vertical sep=0.8cm,
    ylabels at=edge left,
    xlabels at=edge bottom,
  },
  xmode=log,
  xlabel={$c_g$},
  grid=major,
  width=.35\linewidth,
  height=.24\linewidth,
  label style={font=\scriptsize},
  tick label style={font=\scriptsize},
  title style={font=\scriptsize},
  legend style={font=\scriptsize},
  scaled y ticks=false,
  yticklabel style={/pgf/number format/fixed, /pgf/number format/precision=3},
]
\nextgroupplot[title={MATH}, ylabel={Surplus Regret}, ymin=0.0434443, ymax=0.0712512, legend to name=shared_legend, legend columns=3]
\addplot[
  color={rgb,255:red,44;green,160;blue,44},
  mark=triangle,
  mark size=1.5,
  thick,
] coordinates {
  (0.00316228, 0.0630727)
  (0.00681292, 0.0630727)
  (0.014678, 0.0630727)
  (0.0316228, 0.0630727)
  (0.0681292, 0.0630727)
  (0.14678, 0.0630727)
  (0.316228, 0.0630727)
};
\addlegendentry{$f$-only}
\addplot[
  color={rgb,255:red,214;green,39;blue,40},
  mark=square,
  mark size=1.5,
  thick,
] coordinates {
  (0.00316228, 0.0467157)
  (0.00681292, 0.0503663)
  (0.014678, 0.0582314)
  (0.0316228, 0.0751762)
  (0.0681292, 0.111683)
  (0.14678, 0.190333)
  (0.316228, 0.359781)
};
\addlegendentry{$g$-always}
\addplot[
  color={rgb,255:red,31;green,119;blue,180},
  mark=o,
  mark size=1.5,
  thick,
] coordinates {
  (0.00316228, 0.0478131)
  (0.00681292, 0.0522593)
  (0.014678, 0.0580752)
  (0.0316228, 0.0613263)
  (0.0681292, 0.0642749)
  (0.14678, 0.0630727)
  (0.316228, 0.0630727)
};
\addlegendentry{Pandora's Bidder}
\nextgroupplot[title={RAG}, ymin=0.0670537, ymax=0.120806]
\addplot[
  color={rgb,255:red,44;green,160;blue,44},
  mark=triangle,
  mark size=1.5,
  thick,
] coordinates {
  (0.00316228, 0.104996)
  (0.00681292, 0.104996)
  (0.014678, 0.104996)
  (0.0316228, 0.104996)
  (0.0681292, 0.104996)
  (0.14678, 0.104996)
  (0.316228, 0.104996)
};
\addplot[
  color={rgb,255:red,214;green,39;blue,40},
  mark=square,
  mark size=1.5,
  thick,
] coordinates {
  (0.00316228, 0.0733775)
  (0.00681292, 0.0770281)
  (0.014678, 0.0848932)
  (0.0316228, 0.101838)
  (0.0681292, 0.138344)
  (0.14678, 0.216995)
  (0.316228, 0.386443)
};
\addplot[
  color={rgb,255:red,31;green,119;blue,180},
  mark=o,
  mark size=1.5,
  thick,
] coordinates {
  (0.00316228, 0.0752741)
  (0.00681292, 0.0804621)
  (0.014678, 0.0888411)
  (0.0316228, 0.0974313)
  (0.0681292, 0.106849)
  (0.14678, 0.104996)
  (0.316228, 0.104996)
};
\nextgroupplot[title={EmbedLLM}, ymin=0.0657536, ymax=0.0814466]
\addplot[
  color={rgb,255:red,44;green,160;blue,44},
  mark=triangle,
  mark size=1.5,
  thick,
] coordinates {
  (1e-05, 0.076831)
  (0.0001, 0.076831)
  (0.0003, 0.076831)
  (0.001, 0.076831)
  (0.003, 0.076831)
  (0.01, 0.076831)
  (0.03, 0.076831)
  (0.1, 0.076831)
  (0.3, 0.076831)
};
\addplot[
  color={rgb,255:red,214;green,39;blue,40},
  mark=square,
  mark size=1.5,
  thick,
] coordinates {
  (1e-05, 0.0675998)
  (0.0001, 0.0676898)
  (0.0003, 0.0678898)
  (0.001, 0.0685898)
  (0.003, 0.0705898)
  (0.01, 0.0775898)
  (0.03, 0.0975898)
  (0.1, 0.16759)
  (0.3, 0.36759)
};
\addplot[
  color={rgb,255:red,31;green,119;blue,180},
  mark=o,
  mark size=1.5,
  thick,
] coordinates {
  (1e-05, 0.0675766)
  (0.0001, 0.0675555)
  (0.0003, 0.0676006)
  (0.001, 0.0679616)
  (0.003, 0.0688963)
  (0.01, 0.0715672)
  (0.03, 0.0760648)
  (0.1, 0.0768541)
  (0.3, 0.076831)
};
\nextgroupplot[ylabel={Efficiency Regret}, ymin=0.0219998, ymax=0.0540256]
\addplot[
  color={rgb,255:red,44;green,160;blue,44},
  mark=triangle,
  mark size=1.5,
  thick,
] coordinates {
  (0.00316228, 0.0446062)
  (0.00681292, 0.0446062)
  (0.014678, 0.0446062)
  (0.0316228, 0.0446062)
  (0.0681292, 0.0446062)
  (0.14678, 0.0446062)
  (0.316228, 0.0446062)
};
\addplot[
  color={rgb,255:red,214;green,39;blue,40},
  mark=square,
  mark size=1.5,
  thick,
] coordinates {
  (0.00316228, 0.0257676)
  (0.00681292, 0.0294182)
  (0.014678, 0.0372833)
  (0.0316228, 0.0542281)
  (0.0681292, 0.0907345)
  (0.14678, 0.169385)
  (0.316228, 0.338833)
};
\addplot[
  color={rgb,255:red,31;green,119;blue,180},
  mark=o,
  mark size=1.5,
  thick,
] coordinates {
  (0.00316228, 0.0267691)
  (0.00681292, 0.0312194)
  (0.014678, 0.0360174)
  (0.0316228, 0.0389545)
  (0.0681292, 0.0442834)
  (0.14678, 0.0446062)
  (0.316228, 0.0446062)
};
\nextgroupplot[ymin=0.0326347, ymax=0.0672939]
\addplot[
  color={rgb,255:red,44;green,160;blue,44},
  mark=triangle,
  mark size=1.5,
  thick,
] coordinates {
  (0.00316228, 0.0571)
  (0.00681292, 0.0571)
  (0.014678, 0.0571)
  (0.0316228, 0.0571)
  (0.0681292, 0.0571)
  (0.14678, 0.0571)
  (0.316228, 0.0571)
};
\addplot[
  color={rgb,255:red,214;green,39;blue,40},
  mark=square,
  mark size=1.5,
  thick,
] coordinates {
  (0.00316228, 0.0367123)
  (0.00681292, 0.0403629)
  (0.014678, 0.048228)
  (0.0316228, 0.0651728)
  (0.0681292, 0.101679)
  (0.14678, 0.18033)
  (0.316228, 0.349778)
};
\addplot[
  color={rgb,255:red,31;green,119;blue,180},
  mark=o,
  mark size=1.5,
  thick,
] coordinates {
  (0.00316228, 0.0363628)
  (0.00681292, 0.040558)
  (0.014678, 0.047366)
  (0.0316228, 0.0528848)
  (0.0681292, 0.0594931)
  (0.14678, 0.0571)
  (0.316228, 0.0571)
};
\nextgroupplot[ymin=0.0457783, ymax=0.0798449]
\addplot[
  color={rgb,255:red,44;green,160;blue,44},
  mark=triangle,
  mark size=1.5,
  thick,
] coordinates {
  (1e-05, 0.0698253)
  (0.0001, 0.0698253)
  (0.0003, 0.0698253)
  (0.001, 0.0698253)
  (0.003, 0.0698253)
  (0.01, 0.0698253)
  (0.03, 0.0698253)
  (0.1, 0.0698253)
  (0.3, 0.0698253)
};
\addplot[
  color={rgb,255:red,214;green,39;blue,40},
  mark=square,
  mark size=1.5,
  thick,
] coordinates {
  (1e-05, 0.0497861)
  (0.0001, 0.0498761)
  (0.0003, 0.0500761)
  (0.001, 0.0507761)
  (0.003, 0.0527761)
  (0.01, 0.0597761)
  (0.03, 0.0797761)
  (0.1, 0.149776)
  (0.3, 0.349776)
};
\addplot[
  color={rgb,255:red,31;green,119;blue,180},
  mark=o,
  mark size=1.5,
  thick,
] coordinates {
  (1e-05, 0.0497742)
  (0.0001, 0.0498388)
  (0.0003, 0.0499709)
  (0.001, 0.050465)
  (0.003, 0.0515499)
  (0.01, 0.0547563)
  (0.03, 0.0614482)
  (0.1, 0.0698872)
  (0.3, 0.0698253)
};
\end{groupplot}
\node[below=1cm] at ($(group c1r2.south)!0.5!(group c3r2.south)$) {\ref{shared_legend}};
\end{tikzpicture}
\caption{Results for the posted-price auction across the three experimental domains for varying refinement cost $c_g$. The VoI-based bidding strategy interpolates between querying $g$ when cheap and declining when expensive, while static baselines either over-invest or under-invest in refinement.\looseness=-1}
\label{fig:bidding-results}
\end{figure}
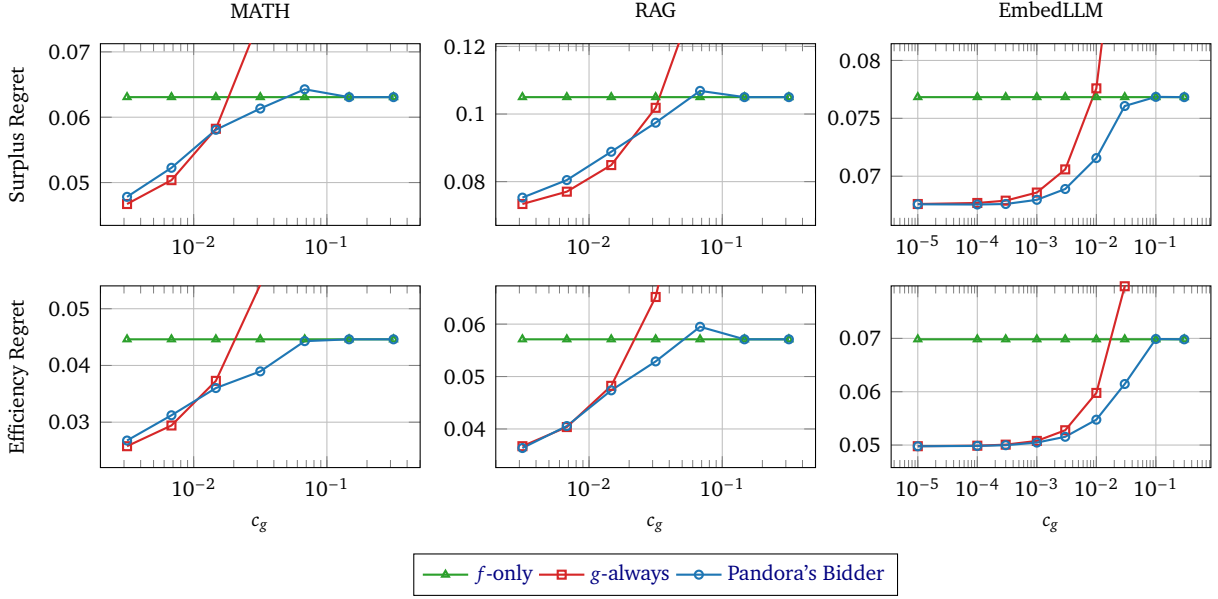

\paragraph{Experimental setup.} We evaluate Pandora's Bidder in a leave-one-out setting across the three domains from \Cref{sec:experimental-domains}. For each query, one specialist is designated as the strategic bidder; the remaining $M - 1$ specialists submit their $g$ estimates nonstrategically, and the platform posts the best as the price via $p = \max_{j \neq m} G_j$. The strategic bidder then applies the VoI-based acceptance strategy above. We rotate the held-out specialist across all $M$ models and report the average. As in the centralized routing experiments, we sweep over $c_g$ to trace out the cost versus performance frontier.\looseness=-1
We measure two quantities: (1)  \emph{specialist surplus}, which is $R_m - p$ when the strategic specialist $m$ accepts the price and $0$ when it declines, as well as any incurred estimation costs; and (2) \emph{allocative efficiency}, which is the true reward $R_{\hat{m}}$ of the winning specialist ($m$ if it accepts, otherwise the platform-chosen specialist),  minus the $m$'s estimation costs. We exclude competitors' estimation costs; by treating their $g$ estimates as given, we isolate the contribution of the VoI-based policy. We compare against the {${g}${-always}} baseline (the bidder always pays for $g$) and the {${f}${-only}} baseline (the bidder always decides based on $f$). For both metrics, we measure the regret versus an oracle that knows $G_m$ for free.\looseness=-1

\paragraph{Results.} \Cref{fig:bidding-results} plots allocative efficiency and bidder surplus for different $c_g$.
Pandora's Bidder closely tracks the lower envelope of the regret of the two baselines across the full $c_g$ range. When $c_g$ is low, the bidder frequently refines its estimate, winning queries where it holds a true advantage. As $c_g$ increases, the bidder selectively defaults to the cheap estimate or declines the price, avoiding the cost collapse of the $g$-always baseline. The $f$-only baseline, conversely, leaves efficiency on the table at low costs by never investing in refinement. With decentralization, however, maximizing local surplus does not always lead to better overall allocative efficiency. For example, when the price is inaccurate (set using the cheap estimate $f$ rather than $g$ for the $M-1$ competitors) the strategic bidder captures higher surplus at the cost of overall allocative efficiency (see additional results in \Cref{sec:supplement-weak-opponent}). Overpriced queries offer insufficient expected profit to incentivize the bidder; declining them protects the bidder's surplus but assigns the input to a less capable specialist. This tension is absent in the centralized setting, where the router internalizes all costs and optimizes global welfare directly. Extending Pandora's Bidder to multi-round ascending-price mechanisms~\citep{parkes2005auction} may recover some of this efficiency loss.\looseness=-1

\section{Related Work}
\label{sec:related}

\paragraph{Information value theory.} Many scenarios require agents to decide whether to spend effort learning more about the world or to exploit what they already know. An early formalization of this decision problem is \emph{information value theory}~\citep{howard1966information}, which provides the general framework for Equations \eqref{eq:reservation-price} and  \eqref{eq:bidding_gain}.
Bayesian Q Learning is an application of information value theory that quantifies the downstream value of actions that increase the precision of the agent's beliefs about state-action values~\citep{dearden1998bayesian}. While similar in spirit, the core ideas of Bayesian Q Learning cannot easily be applied to model routing because state-action pairs cannot be visited more than once; in this work, however, we show that the form of the model routing decision problem admits specialized efficient algorithms based on Pandora's Box (formalized as Pandora's Router in \Cref{sec:pandora-router}).\looseness=-1

\paragraph{LLM routing.} Routing user inputs within a collection of language models has been proposed as a way to reduce costs while maintaining high accuracy~\citep{chen2023frugalgpt} and to leverage complementary capabilities of heterogenous models~\citep{shnitzer2023large}. There are now several routing benchmarks~\citep[e.g.,][]{hu2024routerbench,feng2026moco}, of which we use EmbedLLM~\citep{zhuang2025embedllm}. Most routing approaches require predicting the answer quality for each model on a given input, using pointwise~\citep{jitkrittum2025universal} or pairwise~\citep{ong2025routellm} data. In all of this prior work, value estimation is treated as a cost-free operation with a fixed error profile. Our contribution in Pandora's Router in \Cref{sec:pandora-router} is to treat value estimation as a decision involving a cost-performance tradeoff of its own.\looseness=-1

\paragraph{Adaptive retrieval and tool use.} Closely related to routing is the decision of whether to invoke external tools or retrievers. Adaptive RAG frameworks~\citep{jeong2024adaptive} dynamically adjust retrieval strategies depending on the query complexity to save computational costs. Similarly, recent evaluations of adaptive retrieval~\citep{moskvoretskii2025adaptive} highlight the tension between using internal LLM uncertainty estimation vs. lightweight external heuristics. We formalize this exact tradeoff---that is, deciding whether to spend computational power on self-assessment or tool-planning---through a value-of-information perspective that is both theoretically principled and empirically effective.

\paragraph{Pandora's Box and optimal search.} Prior work on sequential search with costly inspection is reviewed in \S\ref{sec:pandora}. The Pandora's Box problem is also used as a motivating framework for LLM reasoning about cost-uncertainty tradeoffs by \citet{ding2026calibrate}, as well as by \citet{xie2024cost} for cost-aware Bayesian optimization in which the reservation price (for obligatory inspection) is used to drive an acquisition function. In contrast, we instantiate the non-obligatory variant of the Pandora's Box algorithm to drive LLM routing specifically, and we derive closed-form value-of-information expressions under a practical Gaussian signal model that make our algorithm simple to implement and execute. On a more technical level, we focus on a \emph{committing policy} approximation to the non-obligatory inspection variant of the Pandora's Box problem~\citep{beyhaghi2019pandora}. Another class of approximations uses randomization~\citep{beyhaghi2023pandora,scully2024local}, which we may consider in future work.\looseness=-1

\paragraph{Decentralized AI and multi-agent systems.} Decentralized coordination through market mechanisms has a long history in AI, from ContractNet~\citep{davis1983negotiation}, which introduced negotiation-based task allocation, to AgentNet~\citep{yang2025agentnet}, which builds dynamic graph topologies between LLM agents. Our work on Pandora's Bidder in \Cref{sec:pandora-bidder} is most informed by auction-based approaches, which were applied to ContractNet by \citet{sandholm1993implementation} and linked to decentralized RL by \citet{chang2020decentralized}. The combination of auctions with cost-accuracy tradeoffs in value estimation is generally intractable, but \citet{parkes2005auction} offers empirically-validated heuristics.
For a survey on the problem of delegation between AI components, see \citet{tomavsev2026intelligent}. 
Auctions and related mechanisms have received increasing attention as an approach for decentralized orchestration of LLMs~\citep{dutting2024mechanism,collina2025emergent,Zhao2025LLMAuctionGA,tarasova2025decentralized,drwell}, but these frameworks generally assume that agents know their own valuations \textit{a priori}, while we describe how agents can reason about costly value estimation when they are unknown.\looseness=-1
\section{Conclusion}
Value-based routing is a natural approach to optimizing the allocation of computing power to AI model inputs. We have argued that value estimation itself brings cost-accuracy tradeoffs, and we show how the Pandora's Box problem offers a unified framework for value-based routing under heterogeneous value estimators. We also extend the framework to a decentralized setting, in which specialists reason about the value of information in an auction-based allocation mechanism.

\paragraph{Limitations.}
Several limitations suggest promising directions for future work. The Gaussian signal model, while tractable, may not capture the heavy tails or multimodality present in some domains (see \Cref{sec:non-gaussian-signal model} for coverage statistics and a pilot study of a non-Gaussian signal model). The two-estimator restriction ($f$ and $g$) could be relaxed to chains or trees of estimators with varying performance tradeoffs. And the myopic VoI computation in the leave-one-out auction does not account for strategic anticipation of future bids; an extension to a multi-round ascending-price auction~\citep{parkes2005auction} is a natural next step. Furthermore, as discussed in \Cref{sec:pandora-bidder}, when bidding against weaker adversaries that place poor bids, Pandora's Bidder can improve its own utility at the expense of the overall welfare.\looseness=-1

\section*{Acknowledgements}
We thank Alekh Agarwal, Jonathan Berant, and Ian Gemp for helpful research discussions, and Chris Dyer and Artem Sokolov for helpful comments and feedback on the manuscript. This research also benefited from early-stage discussions with Will Dabney.

\bibliography{main}

\clearpage
\appendix
\setcounter{theorem}{0}
\numberwithin{theorem}{section}
\addcontentsline{toc}{chapter}{Appendices}
\etocsetnexttocdepth{2}
\localtableofcontents
\clearpage

\section{Pandora's Routing Objective}
\label{app:pandoras_objective}

In \Cref{sec:pandora-router}, we connect routing to Pandora's Box by treating each specialist as a box and the costly value estimate as the value revealed by opening that box. The revealed value, however, is not $R_m$ itself, but $G_m$, which is the costly, more accurate estimate of $R_m$. In this section we formally justify the objective of searching for the maximum value of $G_m$ when $R_m$ is unknown.

Let $\mathcal I_0$ denote the zero-cost information available before any costly queries, including the prompt and the cheap estimates $\mathbf F$. Inspecting specialist $m$ reveals costly information $Z_m$ at cost $c_m$. In the ideal calibrated case, the value revealed by opening box $m$ is
\begin{equation}
\label{eq:ideal-posterior-value}
G_m^\star
=
\E[R_m \mid \mathcal I_0,Z_m].
\end{equation}
Thus $G_m^\star$ is the posterior decision value of specialist $m$ after $Z_m$ has been acquired.\looseness=-1

\noindent \textbf{Admissible Policies.} We restrict attention to sequential policies that only use information that has actually been observed. Let  $\mathcal F_0=\mathcal I_0$ be the zero-cost information state, and let $O_0=\emptyset$ be the set of opened specialists. At each step $t$, a policy either stops, or selects an unopened specialist $A_{t+1}\in \mathcal M\setminus O_t$ to inspect. The stopping
decision and the choice of $A_{t+1}$ are required to be
$\mathcal F_t$-measurable.  If $A_{t+1}$ is inspected, the policy pays cost $c_{A_{t+1}}$, observes $Z_{A_{t+1}}$, and the  state updates to\looseness=-1
\[
O_{t+1}=O_t\cup\{A_{t+1}\},
\qquad
\mathcal F_{t+1}
=
\mathcal F_t
\vee
\sigma(A_{t+1},Z_{A_{t+1}}).
\]
where $\sigma(A_{t+1},Z_{A_{t+1}})$ denotes the sigma algebra generated by $A_{t+1}$ and $Z_{A_{t+1}}$. Let $\tau$ be the stopping time, $O=O_\tau$ the final opened set, and
\[
C_O=\sum_{i\in O}c_i
\]
the total inspection cost. The final selected specialist $\widehat M$ must be measurable with respect to the terminal information $\mathcal F_\tau$. We call such a policy \emph{admissible}. In the obligatory-inspection case we require
$\widehat M\in O$ almost surely; in the non-obligatory case, $\widehat M$ may also be unopened.

\begin{remark}
Here $\mathcal I_0$ should be interpreted as the information available at zero computational cost. Costly computations are represented by $Z_m$ and are not included in the information state until specialist $m$ is inspected, even when they are deterministic functions of the prompt.
\end{remark}

\begin{proposition}[Opened-value reduction]
\label{prop:opened-value-reduction}
Suppose the collection $\{(R_m,Z_m)\}_{m=1}^M$ is conditionally independent
given $\mathcal I_0$. Let $\Pi$ be any admissible policy that opens a random set
$O$, incurs cost $C_O=\sum_{i\in O}c_i$, and selects an opened specialist
$\widehat M\in O$ almost surely. Define
\[
G_m^\star=\E[R_m\mid \mathcal I_0,Z_m].
\]
Then
\[
\E[R_{\widehat M}-C_O]
=
\E[G_{\widehat M}^\star-C_O].
\]
Consequently, among policies that select opened specialists, maximizing
expected realized reward is equivalent to maximizing the opened posterior
decision value net of inspection costs.
\end{proposition}

\begin{proof}
By admissibility, $\widehat M$ and $C_O$ are $\mathcal F_\tau$-measurable. Hence the tower property gives
\[
\E[R_{\widehat M}-C_O]
=
\E\!\left[
\E[R_{\widehat M}\mid \mathcal F_\tau]-C_O
\right].
\]
Because $\widehat M$ is $\mathcal F_\tau$-measurable,
\[
\E[R_{\widehat M}\mid \mathcal F_\tau]
=
\sum_{m=1}^M
\mathbf 1\{\widehat M=m\}
\E[R_m\mid \mathcal F_\tau].
\]
On the event $\{\widehat M=m\}$, the specialist $m$ has been opened, so $\mathcal F_\tau$ contains $Z_m$. The remaining terminal information consists of zero-cost information and signals from other opened specialists, together with decisions that are measurable functions of those observed signals.
Conditional independence implies that this additional information does not change the posterior mean of $R_m$ once $(\mathcal I_0,Z_m)$ is known. Therefore
\[
\E[R_m\mid \mathcal F_\tau]
=
\E[R_m\mid \mathcal I_0,Z_m]
=
G_m^\star
\]
on the event that $m$ has been opened. Substituting this into the previous equation yields
\[
\E[R_{\widehat M}\mid \mathcal F_\tau]
=
G_{\widehat M}^\star,
\]
and therefore
\[
\E[R_{\widehat M}-C_O]
=
\E[G_{\widehat M}^\star-C_O].
\]
\end{proof}

\begin{corollary}[Sealed value of an unopened specialist]
\label{cor:sealed-value}
Let $\mathcal H$ be any information state generated by an admissible policy before specialist $m$ has been opened. Define the value that would be revealed by opening $m$ at this information state as
\[
G_{m\mid \mathcal H}^\star
=
\E[R_m\mid \mathcal H,Z_m].
\]
If specialist $m$ is selected without being opened, then its decision value is
\[
\E[R_m\mid \mathcal H]
=
\E[G_{m\mid \mathcal H}^\star\mid \mathcal H].
\]
Thus an unopened specialist is evaluated by the posterior expectation of the value that would have been revealed by opening it.
\end{corollary}

\begin{proof}
This is the tower property:
\[
\E[G_{m\mid \mathcal H}^\star\mid \mathcal H]
=
\E\!\left[
\E[R_m\mid \mathcal H,Z_m]
\mid \mathcal H
\right]
=
\E[R_m\mid \mathcal H].
\]
\end{proof}

\Cref{prop:opened-value-reduction} explains why opened boxes can be searched using posterior decision values rather than realized rewards. \Cref{cor:sealed-value}
explains the corresponding role of unopened boxes in the non-obligatory inspection variant: an unopened specialist is a sealed outside option, whose value is the current posterior expectation of the value that opening would reveal. The backup-price construction in Pandora-NI is a way to compare this sealed option against the opened values of the other boxes.

The proposition and corollary are exact for the ideal posterior values. In practice, we do not observe $G_m^\star$ directly. Our costly estimator returns a learned score
\[
G_m = g_m(X,Z_m),
\]
which we use as a plug-in approximation to $G_m^\star$. If $g_m$ were trained by squared loss with unlimited data and sufficient model capacity, its population regression target would be the conditional mean in \Cref{eq:ideal-posterior-value}. With finite data and a restricted model class, $G_m$ is only an approximation. The Pandora reduction should therefore be read as exact for calibrated posterior decision values, and approximate when implemented with learned scores. The empirical question is whether the learned $G_m$ is accurate and calibrated enough to improve routing decisions for the
original reward objective. 
\section{Dataset Details}
\label{app:dataset-details}

\paragraph{MATH.}
For this evaluation, we compile a diverse corpus of mathematical problems spanning multiple levels of difficulty. This includes the standard Hendrycks MATH dataset~\citep{data_HendrycksMATH}, which provides roughly $12500$ problems across seven mathematical domains. We also incorporate the Omni-Math dataset~\citep{data_OmniMath}, an Olympiad-level benchmark designed to assess advanced reasoning capabilities. From this dataset, we utilize a subset of 1625 hard problems, filtered by problems which have between zero and 16 derivation steps. Furthermore, we include 969 American Invitational Mathematics Examination (AIME) problems~\citep{data_aime_competition, data_olympiadbench}, filtered to preclude any overlap with the MATH dataset, alongside 1419 problems from the Harvard-MIT Mathematics Tournament (HMMT)~\citep{data_hmmt_competition} (filtered to avoid overlap with Omni-Math). For the MATH dataset, we retain its standard 5000-problem test split. The remaining datasets are divided using a balanced 50/50 train-test split, resulting in a final aggregate corpus of 9504 training and 7008 testing examples. Prompts are shown in Appendix~\ref{sec:prompts}.\looseness=-1

\paragraph{RAG.} Our RAG evaluation setting builds on prior work that combines two retrieval specialists (Wikipedia and PubMed) along with a low-cost model without retrieval~\citep{eisenstein2025dont}. Following prior work, we train and evaluate on a collection of factoid questions from PopQA~\citep{mallen-etal-2023-trust}, Entity Questions~\citep{sciavolino-etal-2021-simple}, Natural Questions~\citep{kwiatkowski-etal-2019-natural}, and BioASQ~\citep{krithara2023bioasq}. We use 18661 questions for training, 1600 for test, and 400 for calibration, drawing an equal proportion from each of the four datasets. We use the same retrieval pipelines as \citet{eisenstein2025dont}, and generate responses using Gemini-3.1-Flash-Lite. Prompts are shown in Appendix~\ref{sec:prompts}. We use 4 few-shot examples per dataset (16 total) in order to prompt the LLM to generate answers in the right format. The responses are judged for correctness against the ground-truth answers given in each dataset. For each question, we sample $8$ responses and score each for correctness. The reward is the average correctness score, less a constant cost penalty of $0.05$ for systems that use retrieval (i.e., Wikipedia and Pubmed specialists). This cost value was chosen to make the routing problem non-trivial (i.e., the model without retrieval can compete when retrieval is not necessary). To avoid false negatives inherent in exact string matching, we use Gemini-3.1-Flash as an LLM-as-a-judge using the same annotation prompt as \citet{sun-etal-2024-head} (Appendix A.1, Prompt 2), which was measured to have very strong agreement with human judgements ($98\%$). 

\paragraph{EmbedLLM.} 
The EmbedLLM dataset is composed of questions from benchmarks such as GPQA~\citep{rein2023gpqa}, GSM8K~\citep{cobbe2021training}, and MMLU~\citep{hendrycks2020measuring}, along with the scores from 123 models~\citep{zhuang2025embedllm}. We use this dataset without modification. Following \citet{jitkrittum2025universal}, we set the model costs to be a linear multiple of the number of parameters. Specifically we apply a cost of 1 per trillion parameters, so that, e.g., Qwen-1.5-7B-Chat has a cost of $.00772$ and Qwen-1.5-32B-chat has a cost of $.0325$. Seven models were excluded because the parameter count could not be determined (e.g., Claude). We use 16756 prompts for training, 2436 for test, and 609 for calibration.
\section{Implementation details}
\label{sec:implementation_details}

\subsection{Value estimators}

We give additional details for the value estimators $f_m$ and $g_m$ described in \Cref{sec:experimental-domains}.

\paragraph{Embedding-based estimation (KNN).} The cheap value estimator computes an embedding $e(x)$ of the prompt and retrieves the $k$ nearest neighbors $\mathcal{N}_k(x)$ from a calibration set, returning:
\begin{equation}
f_m^{\text{KNN}}(x) = \frac{1}{k}\sum_{x' \in \mathcal{N}_k(x)} R_m(x'),
\end{equation}
where proximity is measured by cosine similarity in the embedding space. This estimator is fast but limited to patterns that are visible in the embedding space. We use Gemini Embedding 2~\citep{lee2025gemini} as our embedder, and use $k = 3$ for the KNN computation.\looseness=-1

\paragraph{Fine-tuned estimation (SFT).} A small language model $h_\theta$ is fine-tuned with a regression loss to predict the cost-adjusted reward from an input context $z_m(x)$:
\begin{equation}
\min_\theta \sum_{(x, m)} \left(h_\theta(z_m(x)) - R_m(x)\right)^2.
\end{equation}
The context $z_m(x)$ can be the prompt alone (SFT-prompt), the prompt with retrieval results (SFT-retrievals), or the prompt with partial reasoning traces (SFT-CoT-$k$). We use Gemini-2.5-Flash-Lite as the base model for fine-tuning. Targets are real-valued. Prompts are shown in Appendix~\ref{sec:prompts}.\looseness=-1

\subsection{Prompts}
\label{sec:prompts}

\begin{Graybox}{Math task: prompt with reasoning.}
\footnotesize
Solve the following math problem. Show your work step-by-step, explaining your reasoning clearly. After your reasoning, provide the final answer enclosed in \textbackslash boxed\{\} tags. \\
\\
Problem: \\
\textcolor{blue}{\{problem\}} \\
\\
Step-by-step solution: \\
$\langle$reasoning here$\rangle$ \\
\\
The final answer is \textbackslash boxed\{answer\}.
\end{Graybox}

\begin{Graybox}{Math task: SFT value estimator prompt using CoT.}
\footnotesize
Question: \textcolor{blue}{\{question\}} \\
Chain-of-Thought: \textcolor{blue}{\{CoT text\}} \\
Specialist: \textcolor{blue}{\{specialist\}} \\
Based on the question and the chain of thought, will the specialist answer correctly? (Predict 1.0 for yes, 0.0 for no)
\end{Graybox}

\begin{Graybox}{RAG task: QA prompt with retrievals.}
\footnotesize
You are a helpful agent whose job is to answer a question. Your answers should be short. For example, if the question is ``What is the capital of France?'', please answer ``Paris'', and not ``Paris is the capital of France''. If you are asked a yes/no question, you may only answer ``yes'' or ``no''. \\
\\
To help you answer the question correctly, you will be given verified information from \textcolor{blue}{\{corpus\}}. The verified information may not be necessary, and you can directly answer the question if you are confident that you have the correct answer. The verified information may also not be relevant or sufficient to answer the question. You should still always respond with your best guess. \\
\\
\textcolor{gray}{\{few\_shot\_examples\}} \\
\\
QUESTION: \textcolor{blue}{\{question\}} \\
\textcolor{brown}{\{retrievals\}} \\
QUESTION: \textcolor{blue}{\{question\}} \\
ANSWER:
\end{Graybox}

\begin{Graybox}{RAG task: standard prompt without retrievals.}
\footnotesize
You are a helpful agent whose job is to answer a question. Your answers should be short. For example, if the question is ``What is the capital of France?'', please answer ``Paris'', and not ``Paris is the capital of France''. If you are asked a yes/no question, you may only answer ``yes'' or ``no''. \\
\\
\textcolor{gray}{\{few\_shot\_examples\}} \\
\\
QUESTION: \textcolor{blue}{\{question\}} \\
ANSWER:
\end{Graybox}
\begin{Graybox}{RAG task: SFT value estimator prompt using retrievals.}
\footnotesize
INSTRUCTIONS: Your task is to predict how likely it is that a language model with \textcolor{blue}{\{capability\}} correctly answers the following question. \\
\\
QUESTION: \textcolor{blue}{\{question\}} \\
\\
INFORMATION: Here are the retrieved passages that the language model will have access to: \\
\\
\textcolor{brown}{1. Title: \{title\_1\} \\
Passage: \{passage\_1\}} \\
\textcolor{brown}{2. Title: \{title\_2\} \\
Passage: \{passage\_2\}} \\
\textcolor{brown}{\ldots} \\
\\
QUESTION: \textcolor{blue}{\{question\}} \\
\\
RESPONSE:
\end{Graybox}

\begin{Graybox}{RAG task: SFT value estimator prompt without retrievals.}
\footnotesize
INSTRUCTIONS: Your task is to predict how likely it is that a language model with \textcolor{blue}{\{capability\}} correctly answers the following question. \\
\\
QUESTION: \textcolor{blue}{\{question\}} \\
\\
RESPONSE:
\end{Graybox}

\begin{Graybox}{EmbedLLM: SFT value estimator prompt (no additional info).}
\footnotesize
INSTRUCTIONS: Your task is to predict how likely it is that the model \textcolor{blue}{\{actor\}} gives a satisfactory response to the following prompt. \\
\\
PROMPT: \textcolor{blue}{\{prompt\}} \\
\\
How likely is it?
\end{Graybox}

\subsection{Experiment compute resources}
Much of our work is based on pre-computed outputs from existing models (e.g., the EmbedLLM data). The main exception is in value estimation, for which we finetuned small Transformer-based encoders ourselves, and inference for the RAG and math domains. Specifically, in the evaluation on the math domain, where we made 15k calls to the Gemini-2.5-Flash-Lite and Gemma3-4B models to compute answers, where the models are served on a TPU cluster with 32 TPUs. This process took less than 1 hour. For the RAG domain, we made  $\sim$500k calls to Gemini-2.5-Flash-Lite to generate 8 responses per question per retrieval setting, and then another $\sim$500k calls to Gemini-2.5-Flash to score all the responses for correctness.  We also fine-tuned small Transformers as value estimators for each setting, requiring 1-3 hours of training time. The training time was higher for the math domain, needing up to 6 hours, while running on 64 Google TPUs utilizing up to 436.69 GiB of Memory. The Pandora's box algorithm itself was run on CPU at minimal cost.
\section{Supplemental Results}

\subsection{Number of queries}
\label{sec:supplement-num-queries}
\Cref{fig:num-queries} shows the number of queries executed by each algorithm. Note that Pandora-OI, the obligatory-inspection variant, must always execute at least one query. At higher costs, most of the savings for Pandora-NI derives from identifying prompts for which it is not necessary to query $g$.

\begin{figure}
\begin{tikzpicture}
\begin{groupplot}[
  group style={
    group size=3 by 1,
    horizontal sep=1cm,
    ylabels at=edge left,
    xlabels at=edge bottom,
  },
  xmode=log,
  xlabel={Cost},
  grid=major,
  width=.35\linewidth,
  height=.3\linewidth,
  label style={font=\scriptsize},
  tick label style={font=\scriptsize},
  title style={font=\scriptsize},
  legend style={font=\scriptsize},
  scaled y ticks=false,
  yticklabel style={/pgf/number format/fixed, /pgf/number format/precision=3},
]
\nextgroupplot[title={MATH}, ylabel={Num Queries}, legend to name=shared_legend, legend columns=5]
\addplot[
  color={rgb,255:red,44;green,160;blue,44},
  mark=triangle,
  mark size=1.5,
  thick,
] coordinates {
  (0.001, 0)
  (0.001778, 0)
  (0.003162, 0)
  (0.005623, 0)
  (0.01, 0)
  (0.01778, 0)
  (0.03162, 0)
  (0.05623, 0)
  (0.1, 0)
};
\addlegendentry{$f$-only}
\addplot[
  color={rgb,255:red,214;green,39;blue,40},
  mark=square,
  mark size=1.5,
  thick,
] coordinates {
  (0.001, 2)
  (0.001778, 2)
  (0.003162, 2)
  (0.005623, 2)
  (0.01, 2)
  (0.01778, 2)
  (0.03162, 2)
  (0.05623, 2)
  (0.1, 2)
};
\addlegendentry{$g$-always}
\addplot[
  color={rgb,255:red,31;green,119;blue,180},
  mark=o,
  mark size=1.5,
  thick,
] coordinates {
  (0.001, 1.28958)
  (0.001778, 0.909167)
  (0.003162, 0.885417)
  (0.005623, 0.797083)
  (0.01, 0.701667)
  (0.01778, 0.434583)
  (0.03162, 0.1825)
  (0.05623, 0.0541667)
  (0.1, 0)
};
\addlegendentry{Pandora NI}
\addplot[
  color={rgb,255:red,255;green,127;blue,14},
  mark=*,
  mark size=1.5,
  thick,
] coordinates {
  (0.001, 1.31542)
  (0.001778, 1.29417)
  (0.003162, 1.27333)
  (0.005623, 1.24542)
  (0.01, 1.17917)
  (0.01778, 1.12)
  (0.03162, 1.08792)
  (0.05623, 1.02917)
  (0.1, 1.01792)
};
\addlegendentry{Pandora OI}
\nextgroupplot[title={RAG}]
\addplot[
  color={rgb,255:red,44;green,160;blue,44},
  mark=triangle,
  mark size=1.5,
  thick,
] coordinates {
  (0.001, 0)
  (0.001778, 0)
  (0.003162, 0)
  (0.005623, 0)
  (0.01, 0)
  (0.01778, 0)
  (0.03162, 0)
  (0.05623, 0)
  (0.1, 0)
};
\addplot[
  color={rgb,255:red,214;green,39;blue,40},
  mark=square,
  mark size=1.5,
  thick,
] coordinates {
  (0.001, 3)
  (0.001778, 3)
  (0.003162, 3)
  (0.005623, 3)
  (0.01, 3)
  (0.01778, 3)
  (0.03162, 3)
  (0.05623, 3)
  (0.1, 3)
};
\addplot[
  color={rgb,255:red,31;green,119;blue,180},
  mark=o,
  mark size=1.5,
  thick,
] coordinates {
  (0.001, 2.26063)
  (0.001778, 2.17125)
  (0.003162, 2.05375)
  (0.005623, 1.9125)
  (0.01, 1.68)
  (0.01778, 1.33563)
  (0.03162, 0.679375)
  (0.05623, 0.02625)
  (0.1, 0)
};
\addplot[
  color={rgb,255:red,255;green,127;blue,14},
  mark=*,
  mark size=1.5,
  thick,
] coordinates {
  (0.001, 2.28)
  (0.001778, 2.20875)
  (0.003162, 2.12312)
  (0.005623, 2.01812)
  (0.01, 1.90063)
  (0.01778, 1.69125)
  (0.03162, 1.54062)
  (0.05623, 1.40437)
  (0.1, 1.22812)
};
\nextgroupplot[title={EmbedLLM},ymode=log]
\addplot[
  color={rgb,255:red,44;green,160;blue,44},
  mark=triangle,
  mark size=1.5,
  thick,
] coordinates {
  (1e-05, 0)
  (3.16228e-05, 0)
  (0.0001, 0)
  (0.000316228, 0)
  (0.001, 0)
  (0.00316228, 0)
  (0.01, 0)
  (0.0316228, 0)
  (0.1, 0)
};
\addplot[
  color={rgb,255:red,214;green,39;blue,40},
  mark=square,
  mark size=1.5,
  thick,
] coordinates {
  (1e-05, 110)
  (3.16228e-05, 110)
  (0.0001, 110)
  (0.000316228, 110)
  (0.001, 110)
  (0.00316228, 110)
  (0.01, 110)
  (0.0316228, 110)
  (0.1, 110)
};
\addplot[
  color={rgb,255:red,31;green,119;blue,180},
  mark=o,
  mark size=1.5,
  thick,
] coordinates {
  (1e-05, 10.8892)
  (3.16228e-05, 9.27375)
  (0.0001, 7.65333)
  (0.000316228, 5.9375)
  (0.001, 4.39167)
  (0.00316228, 2.90667)
  (0.01, 1.18958)
  (0.0316228, 0.254583)
  (0.1, 0)
};
\addplot[
  color={rgb,255:red,255;green,127;blue,14},
  mark=*,
  mark size=1.5,
  thick,
] coordinates {
  (1e-05, 11.3875)
  (3.16228e-05, 9.7525)
  (0.0001, 7.92875)
  (0.000316228, 6.1125)
  (0.001, 4.51667)
  (0.00316228, 3.14125)
  (0.01, 2.1025)
  (0.0316228, 1.28333)
  (0.1, 1.055)
};
\end{groupplot}
\node[below=0.8cm] at ($(group c1r1.south)!0.5!(group c3r1.south)$) {\ref{shared_legend}};
\end{tikzpicture}
\caption{Number of queries to $g$ during model routing.}
\label{fig:num-queries}
\end{figure}
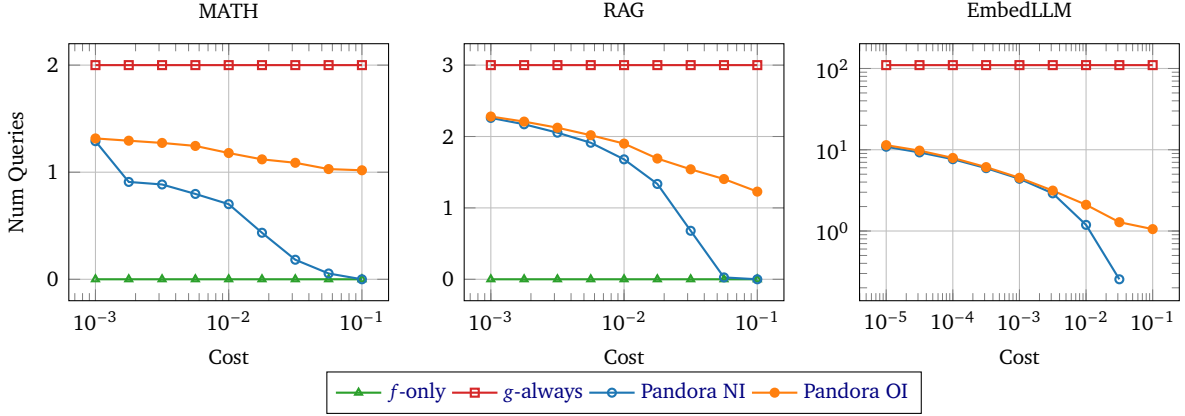

\subsection{Auctions against weak opponents}
\label{sec:supplement-weak-opponent}

In \Cref{sec:pandora-bidder}, the leave-one-out auction computes a posted price from the costly value estimates $g_{j\neq m}$. Since the optimal value-of-information-based strategy for the strategic specialist $m$ depends on the price it faces (i.e., if it is in the interval $[p_\text{lo}, p_{\text{hi}}]$), changing the posted price will also affect the total allocative efficiency and the specialist's surplus.  In particular, we will see that when the price does not accurately represent the true \emph{value} of the best competing specialist, the strategic specialist can exploit this the maximize its own profit at the expense of the overall allocative efficiency of the joint system. \Cref{fig:bidding-results-weak-opponents} shows what happens if the posted prices are computed from the weak estimator $f_m$ instead. In both the RAG and EmbedLLM settings, efficiency regret actually increases, even while individual surplus regret improves at most cost levels. This indicates that the leave-one-out bidder is able to improve its individual utility at the expense of overall welfare, by exploiting poor bids from the other participants. This can happen both when the posted price is too low (the specialist will accept the price, even though a competitor might be better), or when the posted price is too high (the specialist will decline the price, even though it is better than all competitors).

\begin{figure}
\begin{tikzpicture}
\begin{groupplot}[
  group style={
    group size=3 by 2,
    horizontal sep=1cm,
    vertical sep=0.8cm,
    ylabels at=edge left,
    xlabels at=edge bottom,
  },
  xmode=log,
  xlabel={$c_g$},
  grid=major,
  width=.35\linewidth,
  height=.3\linewidth,
  label style={font=\scriptsize},
  tick label style={font=\scriptsize},
  title style={font=\scriptsize},
  legend style={font=\scriptsize},
  scaled y ticks=false,
  yticklabel style={/pgf/number format/fixed, /pgf/number format/precision=3},
]
\nextgroupplot[title={MATH}, ylabel={Surplus Regret}, ymin=0.0429424, ymax=0.0679651, legend to name=shared_legend, legend columns=3]
\addplot[
  color={rgb,255:red,44;green,160;blue,44},
  mark=triangle,
  mark size=1.5,
  thick,
] coordinates {
  (0.00316228, 0.0606055)
  (0.00681292, 0.0606055)
  (0.014678, 0.0606055)
  (0.0316228, 0.0606055)
  (0.0681292, 0.0606055)
  (0.14678, 0.0606055)
  (0.316228, 0.0606055)
};
\addlegendentry{$f$-only}
\addplot[
  color={rgb,255:red,214;green,39;blue,40},
  mark=square,
  mark size=1.5,
  thick,
] coordinates {
  (0.00316228, 0.0458862)
  (0.00681292, 0.0495369)
  (0.014678, 0.0574019)
  (0.0316228, 0.0743467)
  (0.0681292, 0.110853)
  (0.14678, 0.189504)
  (0.316228, 0.358952)
};
\addlegendentry{$g$-always}
\addplot[
  color={rgb,255:red,31;green,119;blue,180},
  mark=o,
  mark size=1.5,
  thick,
] coordinates {
  (0.00316228, 0.0465467)
  (0.00681292, 0.0520571)
  (0.014678, 0.0538013)
  (0.0316228, 0.0569666)
  (0.0681292, 0.0594023)
  (0.14678, 0.0606055)
  (0.316228, 0.0606055)
};
\addlegendentry{Pandora's Bidder}
\nextgroupplot[title={RAG}, ymin=0.052458, ymax=0.183715]
\addplot[
  color={rgb,255:red,44;green,160;blue,44},
  mark=triangle,
  mark size=1.5,
  thick,
] coordinates {
  (0.00316228, 0.14511)
  (0.00681292, 0.14511)
  (0.014678, 0.14511)
  (0.0316228, 0.14511)
  (0.0681292, 0.14511)
  (0.14678, 0.14511)
  (0.316228, 0.14511)
};
\addplot[
  color={rgb,255:red,214;green,39;blue,40},
  mark=square,
  mark size=1.5,
  thick,
] coordinates {
  (0.00316228, 0.0679)
  (0.00681292, 0.0715506)
  (0.014678, 0.0794157)
  (0.0316228, 0.0963605)
  (0.0681292, 0.132867)
  (0.14678, 0.211518)
  (0.316228, 0.380965)
};
\addplot[
  color={rgb,255:red,31;green,119;blue,180},
  mark=o,
  mark size=1.5,
  thick,
] coordinates {
  (0.00316228, 0.0681504)
  (0.00681292, 0.0733534)
  (0.014678, 0.0834884)
  (0.0316228, 0.103149)
  (0.0681292, 0.134762)
  (0.14678, 0.14511)
  (0.316228, 0.14511)
};
\nextgroupplot[title={EmbedLLM}, ymin=0.0877495, ymax=0.100709]
\addplot[
  color={rgb,255:red,44;green,160;blue,44},
  mark=triangle,
  mark size=1.5,
  thick,
] coordinates {
  (1e-05, 0.0968976)
  (0.0001, 0.0968976)
  (0.0003, 0.0968976)
  (0.001, 0.0968976)
  (0.003, 0.0968976)
  (0.01, 0.0968976)
  (0.03, 0.0968976)
  (0.1, 0.0968976)
  (0.3, 0.0968976)
};
\addplot[
  color={rgb,255:red,214;green,39;blue,40},
  mark=square,
  mark size=1.5,
  thick,
] coordinates {
  (1e-05, 0.0892742)
  (0.0001, 0.0893642)
  (0.0003, 0.0895642)
  (0.001, 0.0902642)
  (0.003, 0.0922642)
  (0.01, 0.0992642)
  (0.03, 0.119264)
  (0.1, 0.189264)
  (0.3, 0.389264)
};
\addplot[
  color={rgb,255:red,31;green,119;blue,180},
  mark=o,
  mark size=1.5,
  thick,
] coordinates {
  (1e-05, 0.0892791)
  (0.0001, 0.089381)
  (0.0003, 0.089564)
  (0.001, 0.0900978)
  (0.003, 0.0916965)
  (0.01, 0.0957204)
  (0.03, 0.101644)
  (0.1, 0.0970106)
  (0.3, 0.0968976)
};
\nextgroupplot[ylabel={Efficiency Regret}, ymin=0.0199778, ymax=0.0334951]
\addplot[
  color={rgb,255:red,44;green,160;blue,44},
  mark=triangle,
  mark size=1.5,
  thick,
] coordinates {
  (0.00316228, 0.0295194)
  (0.00681292, 0.0295194)
  (0.014678, 0.0295194)
  (0.0316228, 0.0295194)
  (0.0681292, 0.0295194)
  (0.14678, 0.0295194)
  (0.316228, 0.0295194)
};
\addplot[
  color={rgb,255:red,214;green,39;blue,40},
  mark=square,
  mark size=1.5,
  thick,
] coordinates {
  (0.00316228, 0.021568)
  (0.00681292, 0.0252187)
  (0.014678, 0.0330837)
  (0.0316228, 0.0500285)
  (0.0681292, 0.086535)
  (0.14678, 0.165186)
  (0.316228, 0.334634)
};
\addplot[
  color={rgb,255:red,31;green,119;blue,180},
  mark=o,
  mark size=1.5,
  thick,
] coordinates {
  (0.00316228, 0.0195528)
  (0.00681292, 0.0239116)
  (0.014678, 0.0244236)
  (0.0316228, 0.0256708)
  (0.0681292, 0.028604)
  (0.14678, 0.0295194)
  (0.316228, 0.0295194)
};
\nextgroupplot[ymin=0.0512664, ymax=0.0719147]
\addplot[
  color={rgb,255:red,44;green,160;blue,44},
  mark=triangle,
  mark size=1.5,
  thick,
] coordinates {
  (0.00316228, 0.0658417)
  (0.00681292, 0.0658417)
  (0.014678, 0.0658417)
  (0.0316228, 0.0658417)
  (0.0681292, 0.0658417)
  (0.14678, 0.0658417)
  (0.316228, 0.0658417)
};
\addplot[
  color={rgb,255:red,214;green,39;blue,40},
  mark=square,
  mark size=1.5,
  thick,
] coordinates {
  (0.00316228, 0.0536956)
  (0.00681292, 0.0573463)
  (0.014678, 0.0652113)
  (0.0316228, 0.0821561)
  (0.0681292, 0.118663)
  (0.14678, 0.197313)
  (0.316228, 0.366761)
};
\addplot[
  color={rgb,255:red,31;green,119;blue,180},
  mark=o,
  mark size=1.5,
  thick,
] coordinates {
  (0.00316228, 0.0532318)
  (0.00681292, 0.0561446)
  (0.014678, 0.061764)
  (0.0316228, 0.0700947)
  (0.0681292, 0.0711395)
  (0.14678, 0.0658417)
  (0.316228, 0.0658417)
};
\nextgroupplot[ymin=0.04, ymax=0.12]
\addplot[
  color={rgb,255:red,44;green,160;blue,44},
  mark=triangle,
  mark size=1.5,
  thick,
] coordinates {
  (1e-05, 0.0672329)
  (0.0001, 0.0672329)
  (0.0003, 0.0672329)
  (0.001, 0.0672329)
  (0.003, 0.0672329)
  (0.01, 0.0672329)
  (0.03, 0.0672329)
  (0.1, 0.0672329)
  (0.3, 0.0672329)
};
\addplot[
  color={rgb,255:red,214;green,39;blue,40},
  mark=square,
  mark size=1.5,
  thick,
] coordinates {
  (1e-05, 0.0750623)
  (0.0001, 0.0751523)
  (0.0003, 0.0753523)
  (0.001, 0.0760523)
  (0.003, 0.0780523)
  (0.01, 0.0850523)
  (0.03, 0.105052)
  (0.1, 0.175052)
  (0.3, 0.375052)
};
\addplot[
  color={rgb,255:red,31;green,119;blue,180},
  mark=o,
  mark size=1.5,
  thick,
] coordinates {
  (1e-05, 0.0750614)
  (0.0001, 0.0751256)
  (0.0003, 0.0752598)
  (0.001, 0.0756322)
  (0.003, 0.0767467)
  (0.01, 0.0794477)
  (0.03, 0.0815337)
  (0.1, 0.0674838)
  (0.3, 0.0672329)
};
\end{groupplot}
\node[below=1cm] at ($(group c1r2.south)!0.5!(group c3r2.south)$) {\ref{shared_legend}};
\end{tikzpicture}
\caption{Specialist surplus regret and total allocative efficiency regret of the posted-price auction, when bidding against opponents who have access only to $f$, rather than $g$ as in \Cref{fig:bidding-results}.
}
\label{fig:bidding-results-weak-opponents}
\end{figure}
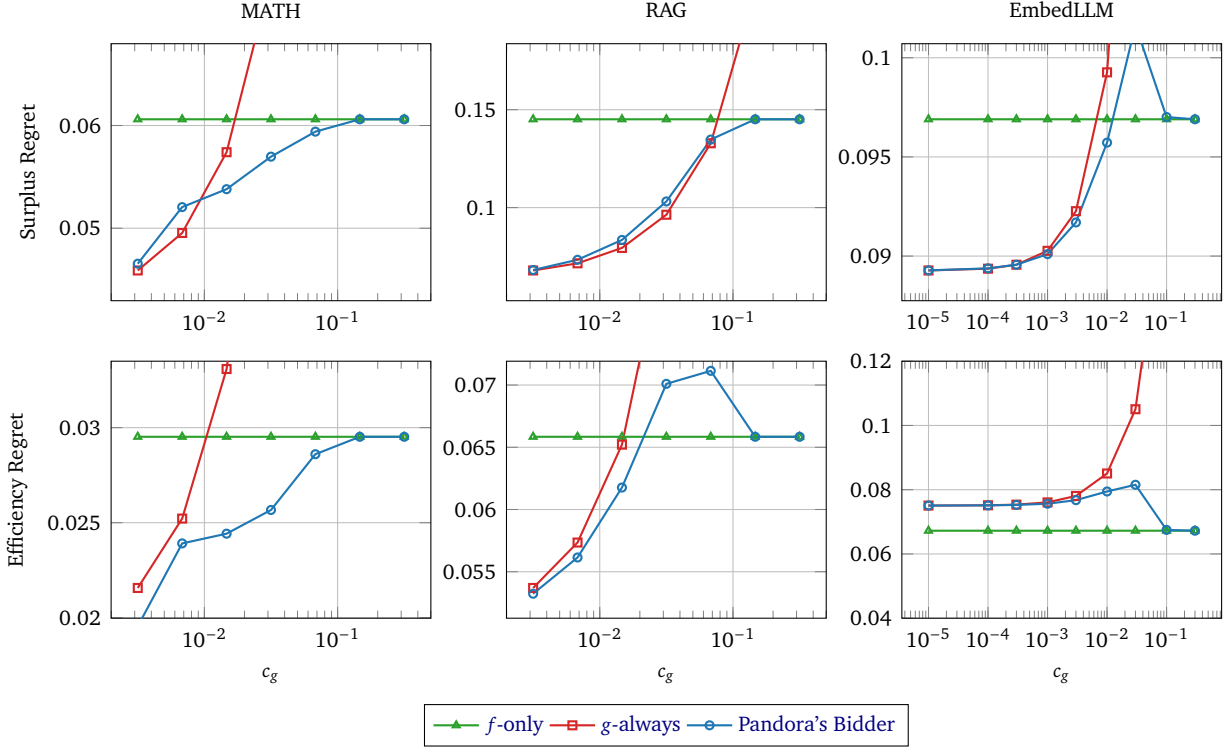

\subsection{Monetary costs for value functions}
\label{sec:supp-monetary-cost-derivations}
As noted in \Cref{sec:setup-value-costs}, we estimate monetary costs from prices listed at \url{https://ai.google.dev/gemini-api/docs/pricing} (retrieved August 1, 2026). In all evaluations, the $f$ value estimator is a small MLP applied to a question embedding. The MLP itself is very cheap to run, and can be assumed to be cost-free. The embedding costs \$0.20 per million tokens, and this can be amortized across target models, although to be conservative about $c_g/c_f$, we do not account for this amortization here. The $g$ value estimators are priced as follows:
\begin{itemize}[leftmargin=*, itemsep=5pt]
    \item \textbf{SFT-CoT-20}, used in the MATH domain, requires running the target model LLM reasoning process for 20 tokens. The cheapest frontier LLM (Gemini 3.5 flash-lite) costs \$2.50 per million output tokens, and \$0.25 per million input tokens. In the MATH domain, there are approximately 43 tokens per question. This gives $c_f \approx 8.6 \times 10^{-6}$ and $c_g \approx 2.5 \times 20 \times 10^{-6} = 5 \times 10^{-5}$, yielding $c_g/c_f \approx 5.8.$\looseness=-1
    \item \textbf{SFT-retrievals}, used in the RAG domain, requires running retrieval, which costs $1.4\times 10^{-2}$ per model evaluation. At $10.1$ tokens per question in the RAG domain, we have $c_f \approx 10.1 \times .2 \times 10^{-6} \approx 2 \times 10^{-6}.$ We conservatively estimate $c_g/c_f >7000.$
    \item \textbf{SFT-prompt}, used in EmbedLLM, requires LLM inference to generate a single token. As above, this costs $2.5 \times 10^{-6}$ per output token and $2.5 \times 10^{-7}$ per input token; the $f$ estimator costs $2\times 10^{-7}$ per input token. At $126$ tokens per query, this yields a ratio of $c_g/c_f \approx 1.6.$ However, this does not account for amortization of the embedding cost $c_f$ across target models, which would be most significant in this domain because it has the largest number of models.
\end{itemize}

\begin{figure}
\centering
\begin{tikzpicture}
\begin{groupplot}[
  group style={
    group size=3 by 1,
    horizontal sep=1cm,
    ylabels at=edge left,
    xlabels at=edge bottom,
  },
  xmode=log,
  grid=major,
  width=.35\linewidth,
  height=.3\linewidth,
  label style={font=\scriptsize},
  tick label style={font=\scriptsize},
  title style={font=\scriptsize},
  legend style={font=\scriptsize},
  scaled y ticks=false,
  yticklabel style={/pgf/number format/fixed, /pgf/number format/precision=3},
]

\nextgroupplot[
  title={MATH},
  xlabel={Inspection cost (\$/1M queries)},
  ylabel={Routing Regret},
  legend to name=monetary_legend,
  legend columns=4,
]
\addplot[
  color={rgb,255:red,44;green,160;blue,44},
  mark=triangle,
  mark size=2.5,
  thick,
  only marks,
] coordinates {
  (17, 0.117)
};
\addlegendentry{$f$-only}
\addplot[
  color={rgb,255:red,214;green,39;blue,40},
  mark=square,
  mark size=2.5,
  thick,
  only marks,
] coordinates {
  (100, 0.090)
};
\addlegendentry{$g$-always}
\addplot[
  color={rgb,255:red,31;green,119;blue,180},
  mark=o,
  mark size=1.5,
  thick,
] coordinates {
  (17, 0.1170)
  (20, 0.1110)
  (26, 0.1062)
  (39, 0.1033)
  (52, 0.0990)
  (57, 0.0965)
  (61, 0.0952)
  (63, 0.0944)
  (82, 0.0927)
};
\addlegendentry{Pandora's Router}
\addplot[
  color={rgb,255:red,148;green,103;blue,189},
  mark=diamond,
  mark size=1.5,
  thick,
] coordinates {
  (17, 0.1170)
  (20, 0.1110)
  (26, 0.1062)
  (39, 0.1033)
  (52, 0.0970)
  (57, 0.0955)
  (61, 0.0952)
  (63, 0.0944)
  (82, 0.0907)
};
\addlegendentry{Margin-$N_{pr}$}

\nextgroupplot[
  title={RAG},
  xlabel={Inspection cost (\$/1K queries)},
  xmin=0.005,
]
\addplot[
  color={rgb,255:red,44;green,160;blue,44},
  mark=triangle,
  mark size=2.5,
  thick,
  only marks,
] coordinates {
  (0.006, 0.150)
};
\addplot[
  color={rgb,255:red,214;green,39;blue,40},
  mark=square,
  mark size=2.5,
  thick,
  only marks,
] coordinates {
  (42, 0.084)
};
\addplot[
  color={rgb,255:red,31;green,119;blue,180},
  mark=o,
  mark size=1.5,
  thick,
] coordinates {
  (0.006, 0.1520)
  (0.37, 0.1515)
  (9.52, 0.1225)
  (18.70, 0.1023)
  (23.53, 0.0932)
  (26.78, 0.0892)
  (28.76, 0.0865)
  (30.40, 0.0861)
  (31.65, 0.0857)
};
\addplot[
  color={rgb,255:red,148;green,103;blue,189},
  mark=diamond,
  mark size=1.5,
  thick,
] coordinates {
  (0.006, 0.1500)
  (0.37, 0.1495)
  (9.52, 0.1325)
  (18.70, 0.1133)
  (23.53, 0.1082)
  (26.78, 0.1032)
  (28.76, 0.1035)
  (30.40, 0.0991)
  (31.65, 0.0947)
};

\nextgroupplot[
  title={EmbedLLM},
  xlabel={Inspection cost (\$/1K queries)},
]
\addplot[
  color={rgb,255:red,44;green,160;blue,44},
  mark=triangle,
  mark size=2.5,
  thick,
  only marks,
] coordinates {
  (2.77, 0.393)
};
\addplot[
  color={rgb,255:red,214;green,39;blue,40},
  mark=square,
  mark size=2.5,
  thick,
  only marks,
] coordinates {
  (3.74, 0.370)
};
\addplot[
  color={rgb,255:red,31;green,119;blue,180},
  mark=o,
  mark size=1.5,
  thick,
] coordinates {
  (2.77, 0.3920)
  (2.78, 0.3949)
  (2.81, 0.3851)
  (2.87, 0.3798)
  (2.92, 0.3776)
  (2.97, 0.3741)
  (3.03, 0.3712)
  (3.09, 0.3737)
  (3.14, 0.3719)
};
\addplot[
  color={rgb,255:red,148;green,103;blue,189},
  mark=diamond,
  mark size=1.5,
  thick,
] coordinates {
  (2.77, 0.3930)
  (2.78, 0.3929)
  (2.81, 0.3911)
  (2.87, 0.3888)
  (2.92, 0.3826)
  (2.97, 0.3761)
  (3.03, 0.3762)
  (3.09, 0.3777)
  (3.14, 0.3769)
};

\end{groupplot}
\node[below=1cm] at ($(group c1r1.south)!0.5!(group c3r1.south)$) {\ref{monetary_legend}};
\end{tikzpicture}
\caption{Routing performance versus monetary inspection cost, as estimated in \Cref{sec:supp-monetary-cost-derivations}.}
\label{fig:monetary-cost-vs-regret}
\end{figure}
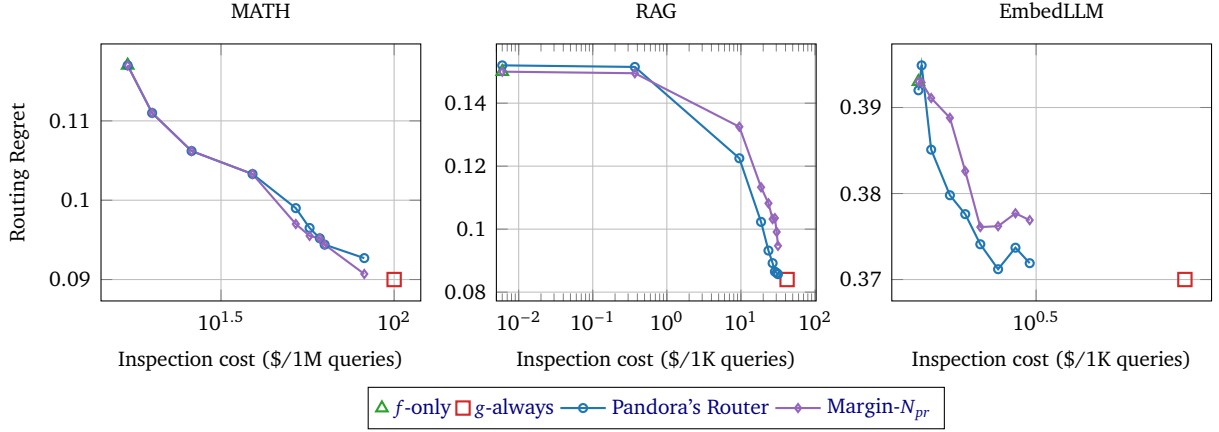

\subsection{Numerical results and significance tests}
\label{sec:supplemental-numerical-significance}
Numerical results and significance tests for the routing evaluations are shown in Tables~\ref{tab:math-routing-significance}, \ref{tab:rag-routing-significance}, and \ref{tab:embedllm-routing-significance}. In each table, bold indicates the lowest regret + inspection cost, asterisk indicates a statistically significant improvement over all alternatives at $p<.05$ by a paired bootstrap test. These results present a consistent picture: Pandora's router is in the argmin of regret at nearly every cost level. At low costs, it matches $g$-always; at high costs it matches $f$-only; at intermediate costs, it often offers the best regret, although this difference is usually not statistically significant with respect to all alternatives.\looseness=-1

\begin{table}
\centering
\footnotesize
\begin{tabular}{llllllll}
\toprule
Cost & f-only & g-only & Top-$2$ & Coin Flip & Random-$N_{\mathrm{pr}}$ & Margin-$N_{\mathrm{pr}}$ & Pandora's Router \\
\midrule
1.5e-03 & 0.117 & 0.093 & 0.093 & 0.109 & 0.194 & \textbf{0.093} & 0.095 \\
2.1e-03 & 0.117 & \textbf{0.094} & 0.094 & 0.109 & 0.193 & 0.097 & 0.096 \\
3.2e-03 & 0.117 & \textbf{0.096} & 0.096 & 0.110 & 0.198 & 0.098 & 0.098 \\
4.6e-03 & 0.117 & \textbf{0.099} & 0.099 & 0.112 & 0.197 & 0.099 & 0.099 \\
6.8e-03 & 0.117 & 0.104 & 0.104 & 0.114 & 0.193 & 0.101 & \textbf{0.101} \\
1.0e-02 & 0.117 & 0.110 & 0.110 & 0.117 & 0.199 & \textbf{0.104}* & 0.107 \\
1.5e-02 & 0.117 & 0.119 & 0.119 & 0.122 & 0.175 & \textbf{0.108}* & 0.113 \\
2.2e-02 & 0.117 & 0.133 & 0.133 & 0.129 & 0.177 & 0.113 & \textbf{0.111} \\
3.2e-02 & 0.117 & 0.153 & 0.153 & 0.139 & 0.147 & 0.112 & \textbf{0.111} \\
4.6e-02 & 0.117 & 0.183 & 0.183 & 0.154 & 0.130 & \textbf{0.113} & 0.113 \\
6.8e-02 & 0.117 & 0.226 & 0.226 & 0.176 & 0.117 & 0.117 & \textbf{0.109}* \\
\bottomrule\\
\end{tabular}
\caption{Numerical results on MATH routing. Bold indicates lowest regret + inspection cost, asterisk indicates significance at $p<.05$ (paired bootstrap).}
\label{tab:math-routing-significance}
\end{table}

\begin{table}
\centering
\footnotesize
\begin{tabular}{llllllll}
\toprule
Cost & f-only & g-only & Top-$2$ & Coin Flip & Random-$N_{\mathrm{pr}}$ & Margin-$N_{\mathrm{pr}}$ & Pandora's Router \\
\midrule
1.5e-03 & 0.150 & \textbf{0.088} & 0.110 & 0.124 & 0.129 & 0.100 & 0.090 \\
2.1e-03 & 0.150 & 0.090 & 0.112 & 0.125 & 0.134 & 0.103 & 0.090 \\
3.2e-03 & 0.150 & 0.093 & 0.114 & 0.127 & 0.146 & 0.110 & 0.092 \\
4.6e-03 & 0.150 & 0.098 & 0.117 & 0.129 & 0.155 & 0.113 & 0.098 \\
6.8e-03 & 0.150 & 0.104 & 0.121 & 0.132 & 0.155 & 0.119 & 0.104 \\
1.0e-02 & 0.150 & 0.114 & 0.127 & 0.137 & 0.170 & 0.124 & 0.110 \\
1.5e-02 & 0.150 & 0.128 & 0.137 & 0.144 & 0.182 & 0.134 & 0.122 \\
2.2e-02 & 0.150 & 0.148 & 0.150 & 0.155 & 0.198 & 0.143 & 0.136 \\
3.2e-02 & 0.150 & 0.178 & 0.171 & 0.170 & 0.196 & 0.154 & \textbf{0.144} \\
4.6e-02 & \textbf{0.150}* & 0.223 & 0.200 & 0.192 & 0.167 & 0.155 & 0.157 \\
6.8e-02 & \textbf{0.150} & 0.288 & 0.244 & 0.225 & 0.150 & 0.150 & 0.155 \\
\bottomrule\\
\end{tabular}
\caption{Numerical results on RAG routing. Bold indicates lowest regret + inspection cost, asterisk indicates significance at $p<.05$ (paired bootstrap).}
\label{tab:rag-routing-significance}
\end{table}

\begin{table}
\centering
\footnotesize
\begin{tabular}{llllllll}
\toprule
Cost & f-only & g-only & Top-$2$ & Coin Flip & Random-$N_{\mathrm{pr}}$ & Margin-$N_{\mathrm{pr}}$ & Pandora's Router \\
\midrule
1.0e-05 & 0.393 & 0.371 & 0.402 & 0.398 & 0.412 & 0.377 & 0.372 \\
1.0e-04 & 0.393 & 0.381 & 0.402 & 0.403 & 0.421 & 0.377 & 0.372 \\
3.0e-04 & 0.393 & 0.403 & 0.403 & 0.414 & 0.436 & 0.378 & 0.374 \\
1.0e-03 & 0.393 & 0.480 & 0.404 & 0.453 & 0.454 & 0.387 & 0.382 \\
3.0e-03 & 0.393 & 0.700 & 0.408 & 0.563 & 0.469 & 0.398 & 0.390 \\
1.0e-02 & \textbf{0.393} & 1.470 & 0.422 & 0.947 & 0.484 & 0.403 & 0.399 \\
3.0e-02 & \textbf{0.393}* & 3.670 & 0.462 & 2.047 & 0.439 & 0.401 & 0.405 \\
1.0e-01 & \textbf{0.393} & 11.370 & 0.602 & 5.894 & 0.393 & 0.393 & 0.395 \\
3.0e-01 & 0.393 & 33.370 & 1.002 & 16.887 & 0.393 & 0.393 & \textbf{0.391} \\
\bottomrule\\
\end{tabular}
\caption{Numerical results on EmbedLLM routing. Bold indicates lowest regret + inspection cost, asterisk indicates significance at $p<.05$ (paired bootstrap).}
\label{tab:embedllm-routing-significance}
\end{table}

\subsection{Alternative algorithm}
\label{sec:supplemental-algorithm}
We build on prior work on non-obligatory inspection~\cite{beyhaghi2019pandora}, which differs slightly from \Cref{alg:pandora-ni}: instead of using $u^{\text{backup}}_m$ in lines 6, 8, and 9, they use $\mu_m.$ In low-cost settings, we have $u^{\text{backup}}_m < E[V_m] = \mu_m,$ which means that in these settings, our variant of the algorithm is less likely to commit to non-inspection and is therefore more aligned with Pandora-OI. In high-cost settings, $u^{\text{backup}}_m > E[V_m] = \mu_m,$ which means that in these settings, our variant is more likely to commit to non-inspection and is therefore more similar to $g$-always. Empirically, our variant is slightly superior to the original algorithm, as shown in \Cref{fig:supplemental-routing-bk19}.

\begin{figure}
\begin{tikzpicture}
\begin{groupplot}[
  group style={
    group size=3 by 1,
    horizontal sep=1cm,
    ylabels at=edge left,
    xlabels at=edge bottom,
  },
  xmode=log,
  xlabel={Cost},
  grid=major,
  width=.35\linewidth,
  height=.3\linewidth,
  label style={font=\scriptsize},
  tick label style={font=\scriptsize},
  title style={font=\scriptsize},
  legend style={font=\scriptsize},
  scaled y ticks=false,
  yticklabel style={/pgf/number format/fixed, /pgf/number format/precision=3},
  restrict y to domain*=0:1, 
]
\nextgroupplot[title={math}, ylabel={Regret + Inspection Cost}, legend to name=shared_legend, legend columns=5,ymax=.13,xmin=.003]
\addplot[
  color={rgb,255:red,44;green,160;blue,44},
  mark=o,
  mark size=1.5,
  thick,
] coordinates {
  (0.00147, 0.116539)
  (0.00215, 0.116539)
  (0.00316, 0.116539)
  (0.00464, 0.116539)
  (0.00681, 0.116539)
  (0.01, 0.116539)
  (0.01468, 0.116539)
  (0.02154, 0.116539)
  (0.03162, 0.116539)
  (0.04642, 0.116539)
  (0.06813, 0.116539)
};
\addlegendentry{$f$-only}
\addplot[
  color={rgb,255:red,214;green,39;blue,40},
  mark=o,
  mark size=1.5,
  thick,
] coordinates {
  (0.00147, 0.0929178)
  (0.00215, 0.0942778)
  (0.00316, 0.0962978)
  (0.00464, 0.0992578)
  (0.00681, 0.103598)
  (0.01, 0.109978)
  (0.01468, 0.119338)
  (0.02154, 0.133058)
  (0.03162, 0.153218)
  (0.04642, 0.182818)
  (0.06813, 0.226238)
};
\addlegendentry{$g$-always}
\addplot[
  color={rgb,255:red,31;green,119;blue,180},
  mark=o,
  mark size=1.5,
  thick,
] coordinates {
  (0.00147, 0.0959872)
  (0.00215, 0.0972924)
  (0.00316, 0.0998404)
  (0.00464, 0.102524)
  (0.00681, 0.103624)
  (0.01, 0.107521)
  (0.01468, 0.106259)
  (0.02154, 0.114382)
  (0.03162, 0.11522)
  (0.04642, 0.114402)
  (0.06813, 0.111548)
};
\addlegendentry{Pandora-NI (BK)}
\addplot[
  color={rgb,255:red,31;green,119;blue,180},
  mark=*,
  mark size=1.5,
  thick,
] coordinates {
  (0.00147, 0.095531)
  (0.00215, 0.0961895)
  (0.00316, 0.0986561)
  (0.00464, 0.0993538)
  (0.00681, 0.101194)
  (0.01, 0.106346)
  (0.01468, 0.112441)
  (0.02154, 0.110407)
  (0.03162, 0.111617)
  (0.04642, 0.113646)
  (0.06813, 0.115633)
};
\addlegendentry{Pandora NI}
\nextgroupplot[title={RAG},ymax=.18,xmin=.005]
\addplot[
  color={rgb,255:red,44;green,160;blue,44},
  mark=o,
  mark size=1.5,
  thick,
] coordinates {
  (0.00147, 0.150281)
  (0.00215, 0.150281)
  (0.00316, 0.150281)
  (0.00464, 0.150281)
  (0.00681, 0.150281)
  (0.01, 0.150281)
  (0.01468, 0.150281)
  (0.02154, 0.150281)
  (0.03162, 0.150281)
  (0.04642, 0.150281)
  (0.06813, 0.150281)
};
\addplot[
  color={rgb,255:red,214;green,39;blue,40},
  mark=o,
  mark size=1.5,
  thick,
] coordinates {
  (0.00147, 0.088035)
  (0.00215, 0.090075)
  (0.00316, 0.093105)
  (0.00464, 0.097545)
  (0.00681, 0.104055)
  (0.01, 0.113625)
  (0.01468, 0.127665)
  (0.02154, 0.148245)
  (0.03162, 0.178485)
  (0.04642, 0.222885)
  (0.06813, 0.288015)
};
\addplot[
  color={rgb,255:red,31;green,119;blue,180},
  mark=o,
  mark size=1.5,
  thick,
] coordinates {
  (0.00147, 0.0929207)
  (0.00215, 0.093175)
  (0.00316, 0.0980866)
  (0.00464, 0.100352)
  (0.00681, 0.104915)
  (0.01, 0.112509)
  (0.01468, 0.127695)
  (0.02154, 0.135712)
  (0.03162, 0.144994)
  (0.04642, 0.151505)
  (0.06813, 0.149037)
};
    \addplot[color={rgb,255:red,31;green,119;blue,180}, mark=*, mark size=1.5, thick] coordinates {
      (0.00464, 0.0981049) (0.00681, 0.102999) (0.01, 0.108834) (0.01468, 0.122738) (0.02154, 0.135758) (0.03162, 0.144131) (0.04642, 0.153414) (0.06813, 0.152547)
    };
\nextgroupplot[
title={embedllm},  
ymax=.45,
]
\addplot[
  color={rgb,255:red,44;green,160;blue,44},
  mark=o,
  mark size=1.5,
  thick,
] coordinates {
  (1e-05, 0.393312)
  (0.0001, 0.393312)
  (0.0003, 0.393312)
  (0.001, 0.393312)
  (0.003, 0.393312)
  (0.01, 0.393312)
  (0.03, 0.393312)
  (0.1, 0.393312)
  (0.3, 0.393312)
};
\addplot[
  color={rgb,255:red,214;green,39;blue,40},
  mark=o,
  mark size=1.5,
  thick,
] coordinates {
  (1e-05, 0.370964)
  (0.0001, 0.380864)
  (0.0003, 0.402864)
  (0.001, 0.479864)
  (0.003, 0.699864)
  (0.01, 1.46986)
  (0.03, 3.66986)
  (0.1, 11.3699)
  (0.3, 33.3699)
};
\addplot[
  color={rgb,255:red,31;green,119;blue,180},
  mark=o,
  mark size=1.5,
  thick,
] coordinates {
  (1e-05, 0.370227)
  (0.0001, 0.373319)
  (0.0003, 0.375406)
  (0.001, 0.382717)
  (0.003, 0.388889)
  (0.01, 0.397511)
  (0.03, 0.404045)
  (0.1, 0.39607)
  (0.3, 0.391553)
};
    \addplot[color={rgb,255:red,31;green,119;blue,180}, mark=*, mark size=1.5, thick] coordinates {
      (1e-05, 0.371504) (3.16228e-05, 0.374234) (0.0001, 0.371615) (0.000316228, 0.376055) (0.001, 0.382078) (0.00316228, 0.388652) (0.01, 0.397289) (0.0316228, 0.402962) (0.1, 0.39242)
    };
\end{groupplot}
\node[below=1cm] at ($(group c1r1.south)!0.5!(group c3r1.south)$) {\ref{shared_legend}};
\end{tikzpicture}
\caption{Evaluation of the \citet{beyhaghi2019pandora} algorithm, which uses $\mu_m$ instead of $u^{\text{backup}}_m$ for evaluating committing policies.}
\label{fig:supplemental-routing-bk19}
\end{figure}
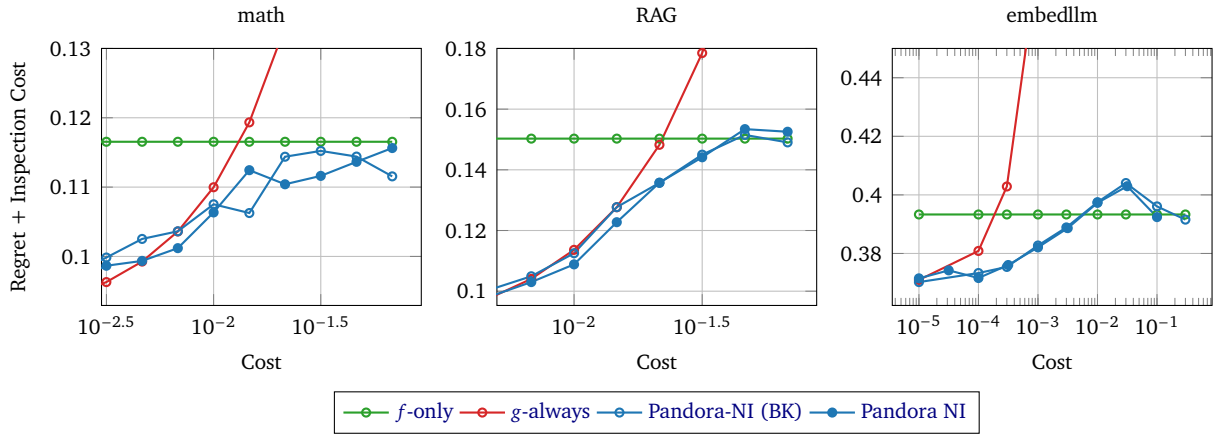

\subsection{Non-Gaussian signal models}
\label{sec:non-gaussian-signal model}
Our implementation models $G_m \mid \mathbf{F} = \mathbf{f}$ as a conditional Gaussian, but real value distributions are bounded in $[0,1]$, so a Gaussian may not be an appropriate model. For this reason, we also experimented using KNN on the vector of $f$-scores to estimate local conditional mean and local standard deviation directly from the neighborhood of $f$. For $K = 16$ this indeed improved calibration metrics somewhat: empirical coverage of the $\pm 1 \sigma$ interval ranged between 58.7\% (RAG) – 78.1\% (MATH) for Gaussian, and 63.0\% (RAG) – 74.8\% (MATH) for KNN (theoretical target: $68.27\%$). For the $\pm 2 \sigma$ interval, the range was $89.8\%$ (RAG) – $94.5\%$ (EmbedLLM) for the Gaussian signal model, and $91.2\%$ (MATH) – $94.7\%$ (EmbedLLM) for KNN (theoretical target: $95.45\%$). 

We apply this signal model to the routing task, with results shown in \Cref{tab:knn-vs-gaussian}. Despite the better calibration of the KNN signal model, we did not see a significant improvement in regret + cost. For this reason, we focus on the Gaussian signal model in the main implementation.

\begin{table}[t]
\centering

\begin{tabular}{ll|cc|cc}
\toprule
\textbf{Setting} & \textbf{Method} & \multicolumn{2}{c|}{\textbf{Gaussian}} & \multicolumn{2}{c}{\textbf{KNN}} \\
& & \textbf{Regret + Cost} & \textbf{Queries} & \textbf{Regret + Cost} & \textbf{Queries} \\
\midrule
MATH & \text{Pandora's Router} & 0.1049 & 0.58 & 0.1090 & 0.78 \\
RAG & \text{Pandora's Router} & 0.1181 & 1.40 & 0.1218 & 1.22 \\
EmbedLLM & \text{Pandora's Router} & 0.3867 & 3.71 & 0.3913 & 0.40 \\
\bottomrule
\end{tabular}
\caption{Comparison between Gaussian and KNN approximations for Pandora's Router.}
\label{tab:knn-vs-gaussian}
\end{table}

\end{document}